\documentclass{svproc}

\usepackage{times}  \usepackage{helvet}  \usepackage{courier}  \usepackage{url}  \usepackage{graphicx}  \usepackage{dsfont}

\usepackage{xspace}
\usepackage{times}
\usepackage{amsmath,amssymb,amsfonts}
\usepackage{verbatim}
\usepackage{graphicx}
\usepackage{psfrag,graphicx,epsfig,epsf}
\usepackage{graphics}
\usepackage{latexsym}
\usepackage{fancyhdr}
\usepackage{setspace}
\usepackage{color}
\usepackage{url}
\usepackage{mathtools}
\usepackage{amsmath,amsfonts,amssymb,mathrsfs}
\usepackage{psfrag,graphicx,epsfig,epsf}
\usepackage[ruled,linesnumbered, noend]{algorithm2e}
\usepackage{fancyhdr}
\usepackage{wrapfig}
\usepackage{subfigure}
\usepackage{mathtools}
\usepackage{url}
\usepackage{color}
\usepackage{verbatim}
\usepackage{xspace}
\usepackage{bbm}

\title{\Large \bf  Synchronized Multi-Arm Rearrangement\\ Guided by Mode Graphs with Capacity Constraints}

\author{Rahul Shome and Kostas E. Bekris}

\authorrunning{Shome and Bekris}
\institute{Rutgers University, New Brunswick, USA}

\newcommand{\Pspace}{\mathcal{P}}

\newcommand{\pose}{p}

\newenvironment{myitem}{\begin{list}{$\bullet$}
{\setlength{\itemsep}{-0pt}
\setlength{\topsep}{0pt}
\setlength{\labelwidth}{0pt}
\setlength{\leftmargin}{10pt}
\setlength{\parsep}{-0pt}
\setlength{\itemsep}{0pt}
\setlength{\partopsep}{0pt}}}{\end{list}}

\newtheorem{observation}{\bf Observation}

\newtheorem{condition}{\bf Condition}

\newcommand{\dof}{{\tt DoF}\xspace}

\newcommand{\tree}{\ensuremath{\mathbb{T} \ }}

\newcommand{\drrtstar}{\ensuremath{{\tt dRRT^*}}}

\newcommand{\cost}{\mathtt{C}}

\newtheorem{assumption}{Assumption}

\newcommand{\objectset}{\mathcal{O}}
\newcommand{\object}{o}
\newcommand{\workspace}{\mathcal{W}}
\newcommand{\taskspace}{\mathcal{T}}
\newcommand{\tfree}{\mathcal{T}_{\rm free}}
\newcommand{\arrangement}{A}

\newcommand{\arm}{m}
\newcommand{\arms}{\mathcal{M}}

\newcommand{\state}{q}

\newcommand{\ainit}{A_{\rm init}}
\newcommand{\atarget}{A_{\rm goal}}

\newcounter{model}

\definecolor{darkgreen}{RGB}{30,150,30}

\newcommand{\cspace}{\mathbb{C}}

\newcommand{\mode}{\text{\textbf{m}}}

\newcommand{\pregion}{\mathcal{S}}
\newcommand{\region}{{S}}

\newcommand{\modegraph}[2]{{\mathds{G}}(#1; #2)}
\newcommand{\modeg}{{\mathds{G}}}
\newcommand{\modenodes}{{\mathds{M}}}
\newcommand{\modeedges}{{\mathds{E}}}
\newcommand{\modes}{\text{\textbf{m}}}
\newcommand{\actions}{\text{\textbf{e}}}
\newcommand{\capacity}{{\mathds{F}}}
\newcommand{\weights}{{\mathds{W}}}
\newcommand{\State}{Q}
\newcommand{\mstate}{\mathds{Q}}
\newcommand{\heuristic}{\mathcal{H}}
\newcommand{\modeinv}{\modenodes^{-1}}
\renewcommand{\tree}{\mathds{T}}
\newcommand{\moma}{\ensuremath{\mathtt{SMAR}}\xspace}
\newcommand{\tamp}{\ensuremath{\mathtt{TAMP}}\xspace}
\newcommand{\mpp}{\ensuremath{\mathtt{MAPF}}\xspace}
\newcommand{\mgcc}{\momaalgo}
\newcommand{\ilp}{\ensuremath{\mathtt{ILP}}\xspace}

\newcommand{\modegoal}{\mathcal{G}\xspace}
\newcommand{\momaalgo}{\ensuremath{\mathtt{SMAR}\mathtt{T}}\xspace}

\newcommand{\rahul}[1]{{#1}} 

\begin{document}

\maketitle
\thispagestyle{empty}
\pagestyle{empty}

\begin{abstract}
Solving task planning problems involving multiple objects and multiple robotic arms poses scalability challenges. Such problems involve not only coordinating multiple high-DoF arms, but also searching through possible sequences of actions including object placements, and handoffs. The current work identifies a useful connection between multi-arm rearrangement and recent results in multi-body path planning on graphs with vertex capacity constraints. 
Solving a synchronized multi-arm rearrangement at a high-level involves reasoning over a modal graph, where nodes correspond to stable object placements and object transfer states by the arms. Edges of this graph correspond to pick, placement and handoff operations. The objects can be viewed as pebbles moving over this graph, which has capacity constraints. For instance, each arm can carry a single object but placement locations can accumulate many objects. 
Efficient integer linear programming-based solvers have been proposed for the corresponding pebble problem. The current work proposes a heuristic to guide the task planning process for synchronized multi-arm rearrangement. Results indicate good scalability to multiple arms and objects, and an algorithm that can find high-quality solutions fast and exhibiting desirable anytime behavior.
\end{abstract}

 
\section{Introduction}
\label{sec:introduction}
Robotic arms are deployed in a variety of applications that involve pick-and-place tasks ranging from manufacturing to logistics and recycling. With the increasing affordability of such systems, and the availability of platforms like dual-arm humanoids, it is important to study how multiple robotic arms can expand upon the capabilities of individual arms and achieve faster execution. In some traditional deployments of multiple arms, such as automotive manufacturing, the environment is perfectly known and each arm performs its own independent task in relatively well-separated workspaces. 

\begin{figure}[t]
    \centering
    \includegraphics[height=0.8in]{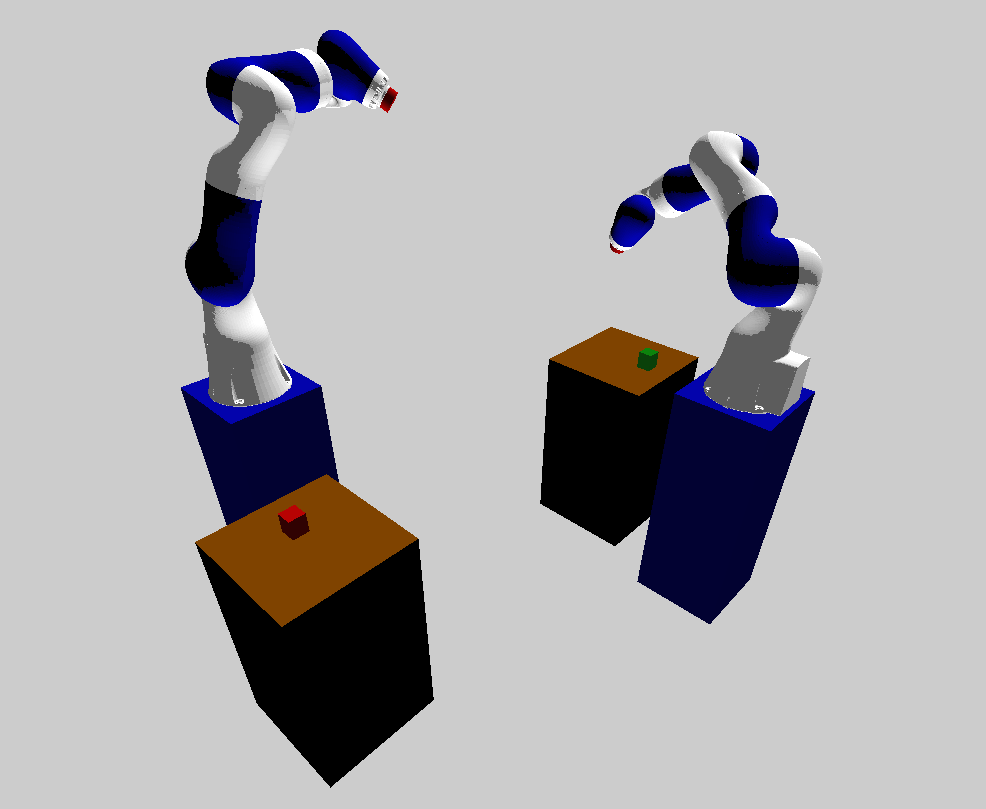}
    \includegraphics[height=0.8in]{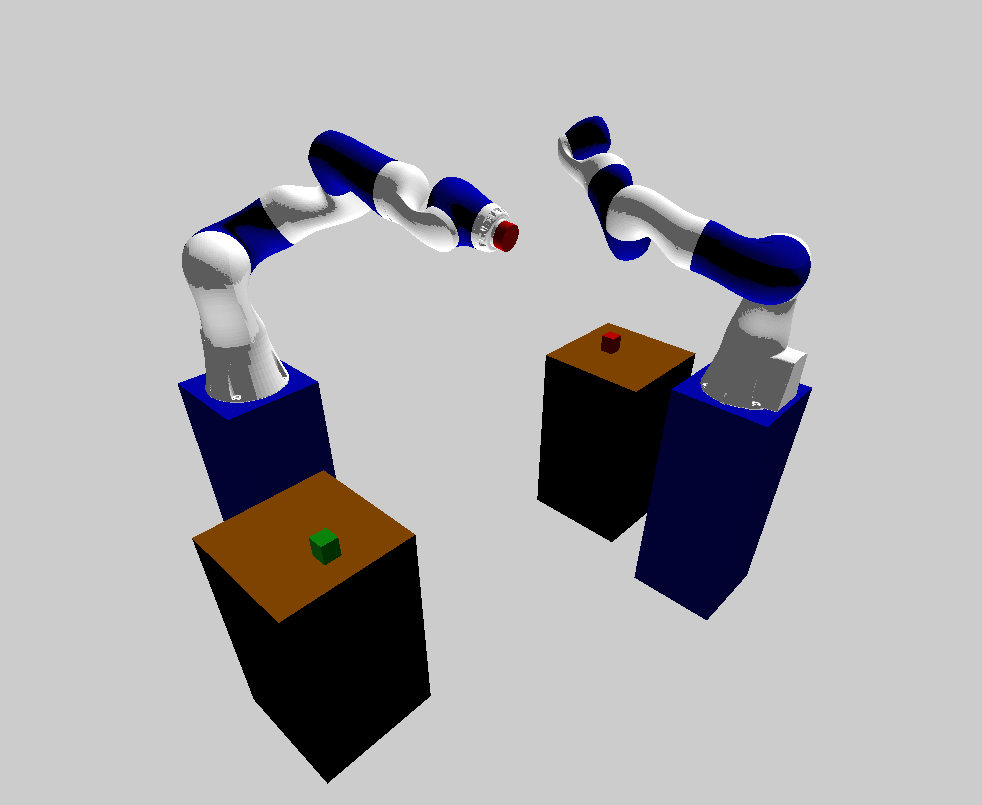}\hspace{0.03\textwidth}
    \includegraphics[height=0.8in]{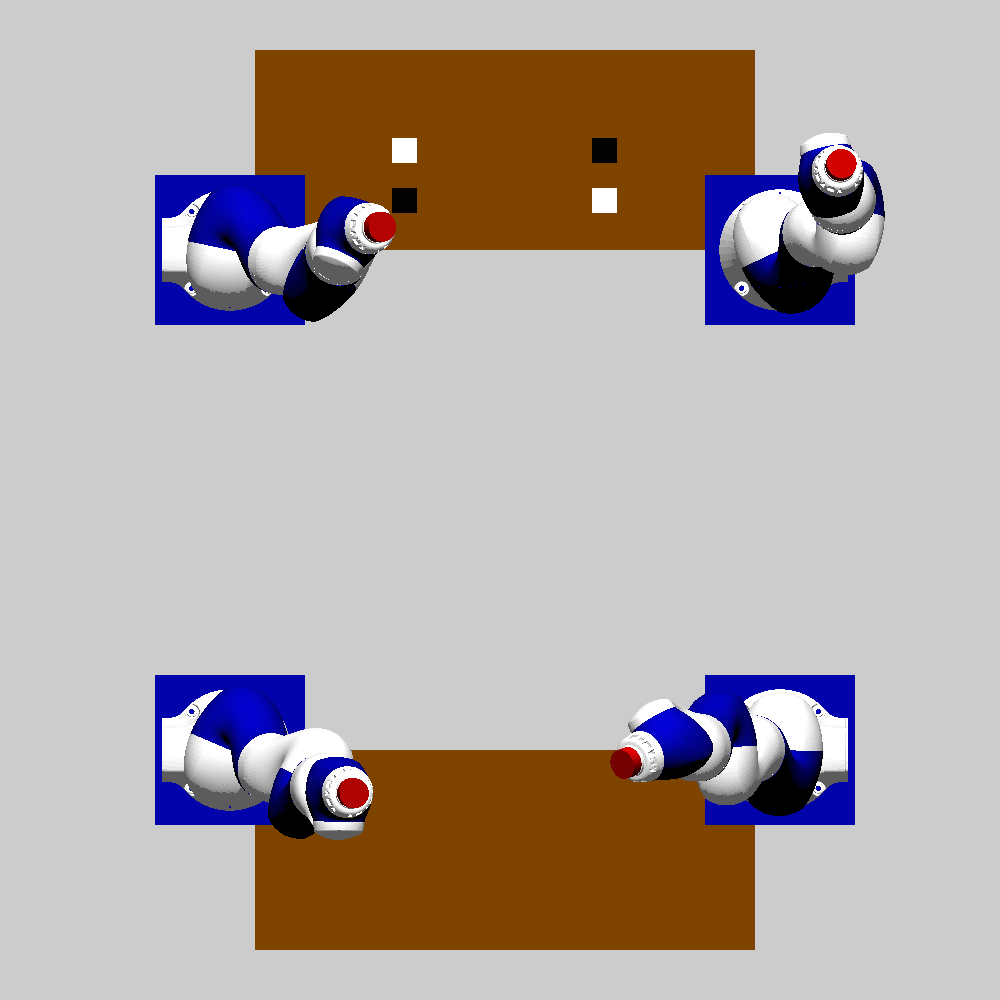}
    \includegraphics[height=0.8in]{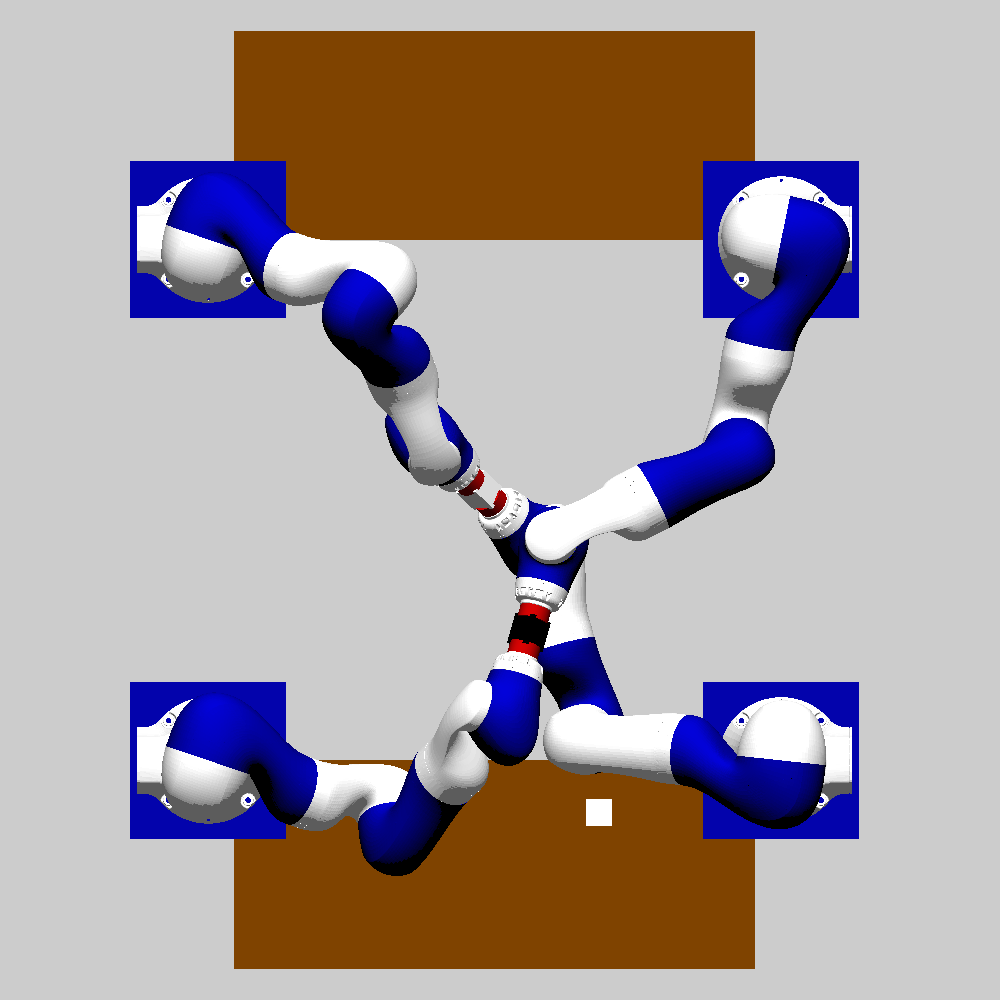}
    \includegraphics[height=0.8in]{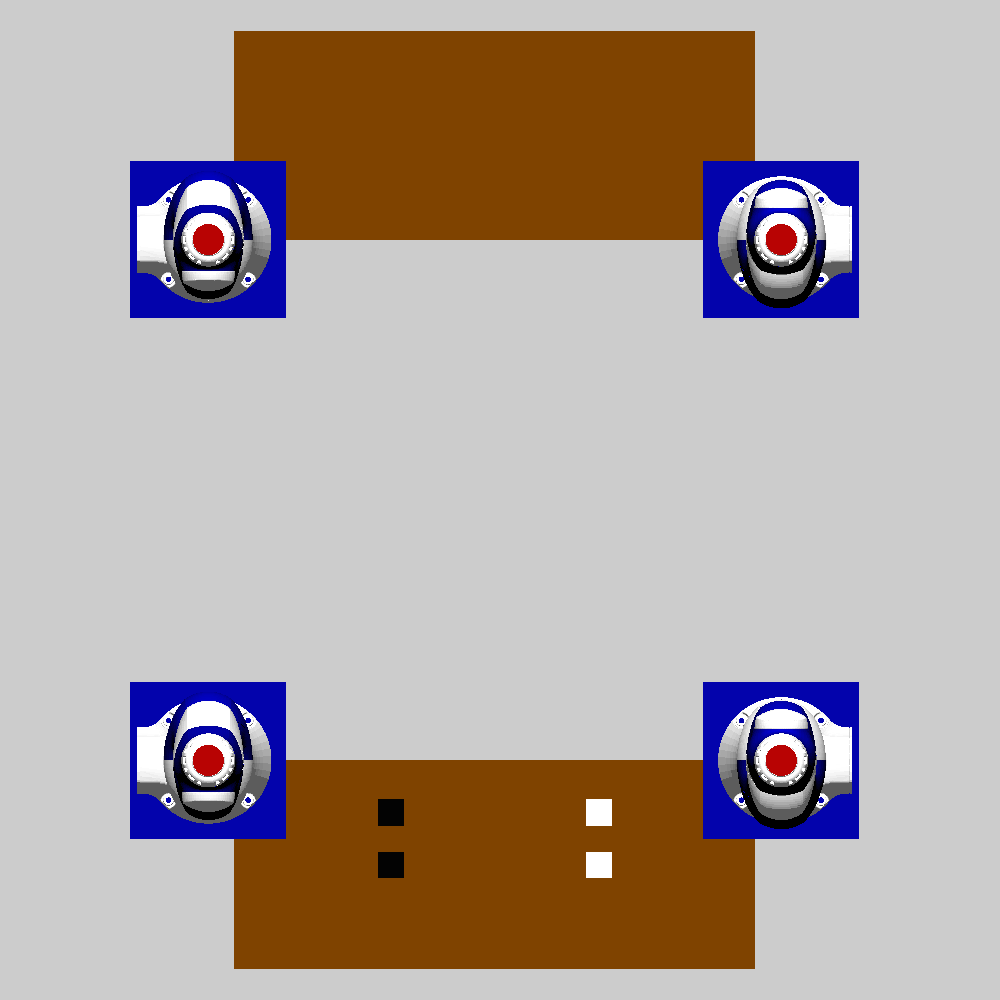}
    \vspace{-.1in}
    \caption{\moma problems: (\textit{Left two: }) A motivating problem involving switching the tables on which the objects lie. Such a problem already needs some high-level guidance to find a solution. (\textit{Right three: }) Steps in a larger problem instance involving 4 arms and 4 objects.}
    \vspace{-0.2in}
    \label{fig:motivation}
\end{figure}

This work focuses on the case that the objects location are not pre-encoded and the arms and not sequestered. Instead the arms need to coordinate to solve object rearrangement tasks, where - beyond picking and placing - handoffs are also required to be performed. Such multi-arm rearrangement challenges are clearly computationally hard. The robotic arms are already high DoF systems. Coordinating multiple such arms to manipulate multiple objects results in an even larger configuration space. Furthermore, the overall planning problem involves searching both the continuous space of each robot and scheduling the discrete sequence of actions, i.e., picking, placing and handoffs.

Consider the example problem shown in  Fig~\ref{fig:motivation}(left), which involves swapping objects between two tables not both reachable by a single arm. A greedy approach for such an object swap -- e.g., the left arm grasps the left object and the right arm grasps the right -- results in a bottleneck, where both arms hold an object and are unable to perform the handoff. The solution is to transfer one of the objects first, and then move the second object. For this example, it may appear that enforcing moving one object at a time is a desirable approach. But this would be highly inefficient for the more general case of Fig~\ref{fig:motivation}(right), which involves more arms and objects. In such cases, it is also desirable to reduce the cost of the solution, typically corresponding to the makespan of the tasks, by simultaneously moving arms and manipulating multiple objects.

The current work focuses on a synchronized version of the problem, where the discrete actions of the arms (picking, placing and handoffs) are synchronized for a subset of the arms, i.e., arms can also be assigned a no-operation action. The paper studies the structure of such synchronized multi-arm rearrangement (\moma), and argues that:

\begin{myitem}
    \item Under assumptions, there is a novel analogy between \moma and multi-agent path finding (\mpp) given a graph abstraction: an object-centric, mode-graph with capacity constraints. Such an abstraction has been recently studied in the \mpp\ literature \cite{surynek2019multi}. Solutions to such problems map well to solutions to a class of \moma problems as long as the underlying motion planning problems can also be solved.
    \item An integer linear programming (\ilp) solution,  defined by building on top of previous work~\cite{yu2016optimal}, is practical and fast for the version of the \mpp over mode graph with capacity that is identified to be appropriate to model \moma problems.
    \item This \ilp solution for the \mpp over mode graph with capacity is effective in guiding the exploration of a forward search tree (\momaalgo) for the problem.
\end{myitem}
Given these observations, the paper demonstrates the applicability to problems involving up to $9$ arms and $4$ objects in simulation, taking at most $10s$ for the harder cases.

\vspace{-.2in}
\section{Related Work}
\vspace{-.15in}

\noindent\textbf{Rearrangement: } Rearrangement planning~\cite{ota2004rearrangement} is a class of manipulation task and motion planning (\tamp) problems. Earlier work focused on efficient solutions to  monotone instances~\cite{stilman2007manipulation}. Efficient solutions to the related assembly planning problem~\cite{Halperin:2000uq} also often assume monotonicity. Over the last decade, the focus incrementally moved to harder \rahul{single-arm} instances of rearrangement and manipulation task planning ~\cite{krontiris2015dealing,krontiris2016efficiently}, \rahul{ and hybrid approaches~\cite{havur2014geometric}. Progress was made in} studying the structure of hard problem instances that lead to efficient solutions~\cite{han2017complexity}. \rahul{The domain of object stacking~\cite{han2018efficient} with a single arm was also explored.} This motivated work on synchronized, dual-arm rearrangement~\cite{shome2018rearrangement}, which managed to map the problem to a sequence of simpler sub-problems. The current work follows the same philosophy to lend structure to a subset of \moma problems that leverages their combinatorial structure and allows for efficiency. \rahul{The object-centric focus of the current work is similar to previous approaches~\cite{han2017complexity,han2018efficient}. The current effort, however, explicitly handles the complexity of dealing with multiple arms and identifies the relationship to capacity constraints on \textit{pebble graphs}.}

\noindent\textbf{Manipulation TAMP: } Early work
 focused on formalizing the problem's multi-modal structure  \cite{Simeon:2004tg,Hauser2011Randomized-Multi-Modal-}. 
 \tamp\ can also be approached in an integrated manner via constrained optimization formulations~\cite{toussaint2015logic}. There are also hierarchical search strategies, which at a low-level call time-budgeted motion planning subroutines, and based on their outcomes, they guide the search~\cite{akbari2018combined,garrett2018ffrob} over actions in the
task space \cite{dantam2016incremental,Kaelbling:2011gb}. Heuristics are important to effectively guide such \tamp algorithms. For instance, in multi-arm manipulation~ an effective heuristic is to consider the path for the object as a free-flying rigid body~\cite{cohen2015planning}. \rahul{More recently multi-arm task planning has been studied using an answer-set programming-based hybrid approach~\cite{saribatur2019finding}. Such hybrid approaches, which take into account symbolic constraints, can address a general set of task planning problems, and can guarantee probabilistic completeness. The current work focuses on the scalability of object rearrangement problems to multi-arm settings and considers aspects of solution quality as well}.
The asymptotic optimality of \tamp\
problems has been investigated~\cite{vega2016asymptotically,schmitt2017optimal}. As the number of robots increases in \tamp, so does the number of modes in the search space of task planning~\cite{dobson2015planning,Harada2014A-Manipulation}. An efficient \tamp\ planner was proposed for single-object, multi-arm manipulation~\cite{shome2019anytime}. The current work builds on top of these efforts \cite{Hauser2011Randomized-Multi-Modal-,vega2016asymptotically,shome2019anytime}. \rahul{
The current work focuses on rearrangement problems involving both multiple arms and multiple objects and aims to identify effective heuristics for an otherwise asymptotically optimal search of the overall search space.
}

\noindent\textbf{Motion Planning: }
Sampling-based approaches have been a popular class of algorithms in
motion planning research ~\cite{Kavraki1996Probabilistic-R,LaValle2001}, including more recently asymptotically optimal (AO)
variants~\cite{Karaman2011Sampling-based-,janson2015fast}.
Recent advances in sampling-based, multi-robot motion planning focused on high \dof\ systems,
such as the \drrtstar\ method \cite{Dobson:2017aa,shome2019drrt},
which effectively decomposes the planning space, while guaranteeing completeness and asymptotic optimality~\cite{SoloveySH16:ijrr}. The current work uses an underlying \drrtstar-like approach to compute simultaneous motions for multiple arms.

\noindent\textbf{Multi-agent Path Finding (\mpp): }
Coupled approaches
\cite{SoloveySH16:ijrr,Wagner:2015bd}
operate in the composite, high-DoF configuration space. They can achieve completeness and optimality in principle but are often computationally intractable. Decoupled solutions~\cite{GhrOkaLav05,Berg:2009ve} reduce the size of the search space by committing to individual agent solutions. They typically lack
completeness and optimality.  Fast, efficient coupled algorithms on graphs have been proposed~\cite{yu2016optimal}, which formalize the optimal multi-agent path finding problem as integer linear program.
Other efforts have focused on solving these problems with SAT solvers~\cite{surynek2016efficient}, and more recently by modeling vertex capacities~\cite{surynek2019multi}.  The current work will draw the relationship between the graphical structures of \mpp problems~\cite{surynek2019multi} and borrows solution frameworks from the corresponding literature~\cite{yu2016optimal} to generate actions sequences for synchronized, multi-arm rearrangement.

\vspace{-0.15in}
\section{Problem Setup and Terminology}
\label{sec:problem}
\vspace{-.1in}

Consider a workspace $\workspace\subset SE(3)$, which contains a set of obstacles, a set of $r$ stationary robot arms $\arms = \{\arm_1,\ldots \arm_r\}$, and a set of $k$ objects $\objectset = \{ \object_1,\ldots\object_k \}$. Each of the arms $\arm_i$ has a state $\state_i$ in a $d_i$-dim. configuration space $\cspace_{i\in[1\ldots r]} \subset \mathbb{R}^{d_i}$. Each object can attain a pose $\pose_j$ in $\Pspace_{j\in[1\ldots k]}\subset SE(3).$ The combined planning space is:
\vspace{-.1in}
$$ \taskspace = \prod_{i\in[1\ldots r]} \cspace_i \times \prod_{j\in[1\ldots k]} \Pspace_j. \vspace{-.1in}$$
Let a task space state then be $\State = (\state_1,\ldots\state_r,\pose_1,\ldots\pose_k) \in \taskspace$.
A subset of this planning space is collision-free given all possible interactions of the arms, objects and obstacles, defined as $\tfree \subset \taskspace$. The arms' end-effectors are allowed to touch the objects for grasping purposes in $\tfree$.

A path for an arm $\arm_i$ is defined as $\pi_i: [0,1] \rightarrow \cspace_i$. Define by $\pi(t_1,t_2)$ the subsequence of the path between $t_1\rightarrow t_2$, where $t_1,t_2 \in [0,1]$.  A composite path for all the arms $\Pi : [0,1] \rightarrow \prod_{i\in[1\ldots r]} \cspace_i$ is defined as the concurrent motion of all the arms. There are $k+1$ task modes of arm motions depending on interactions with the $k$ objects: (i) \textbf{Move mode:} Moves (or transit) paths $\pi^M_i(t_{\rm init}, t_{\rm end})$ are motions of an arm $i$ when no object is carried by the end-effector. (ii) \textbf{Transfer mode:} Transfer paths $\pi_i^T(t_{\rm init}, t_{\rm end})$ are motions of an arm $i$ when an object is carried by the end-effector. There are $k$ such transfer modes i.e., one per object.

A complementary set of $r+1$ task modes exist per object. Each object is constrained by either resting stably on a supporting surface or being grasped by an arm: (i) \textbf{Stable Pose}: A stable pose $\pose_j^{\rm place}$ is a pose for object $j$ that is statically stable given a supporting surface, say a tabletop, which is otherwise an obstacle. (ii) \textbf{Grasped}: An end-effector can maintain a constant relative pose in $SE(3)$ with an object when a \textit{grasp} is engaged.  Manipulation actions that affect the object states above corresponds to \textit{picks, placements}. This is demonstrated in Fig~\ref{fig:taskmodes}. 
Additionally, a special interaction between two arms and one object introduces an additional action \textbf{Handoff}: an instantaneous grasp operation by one arm, and a release operation by another on the same object.

\begin{wrapfigure}{r}{0.5\textwidth}
    \centering
    \vspace{-0.3in}
    \includegraphics[width=0.49\textwidth]{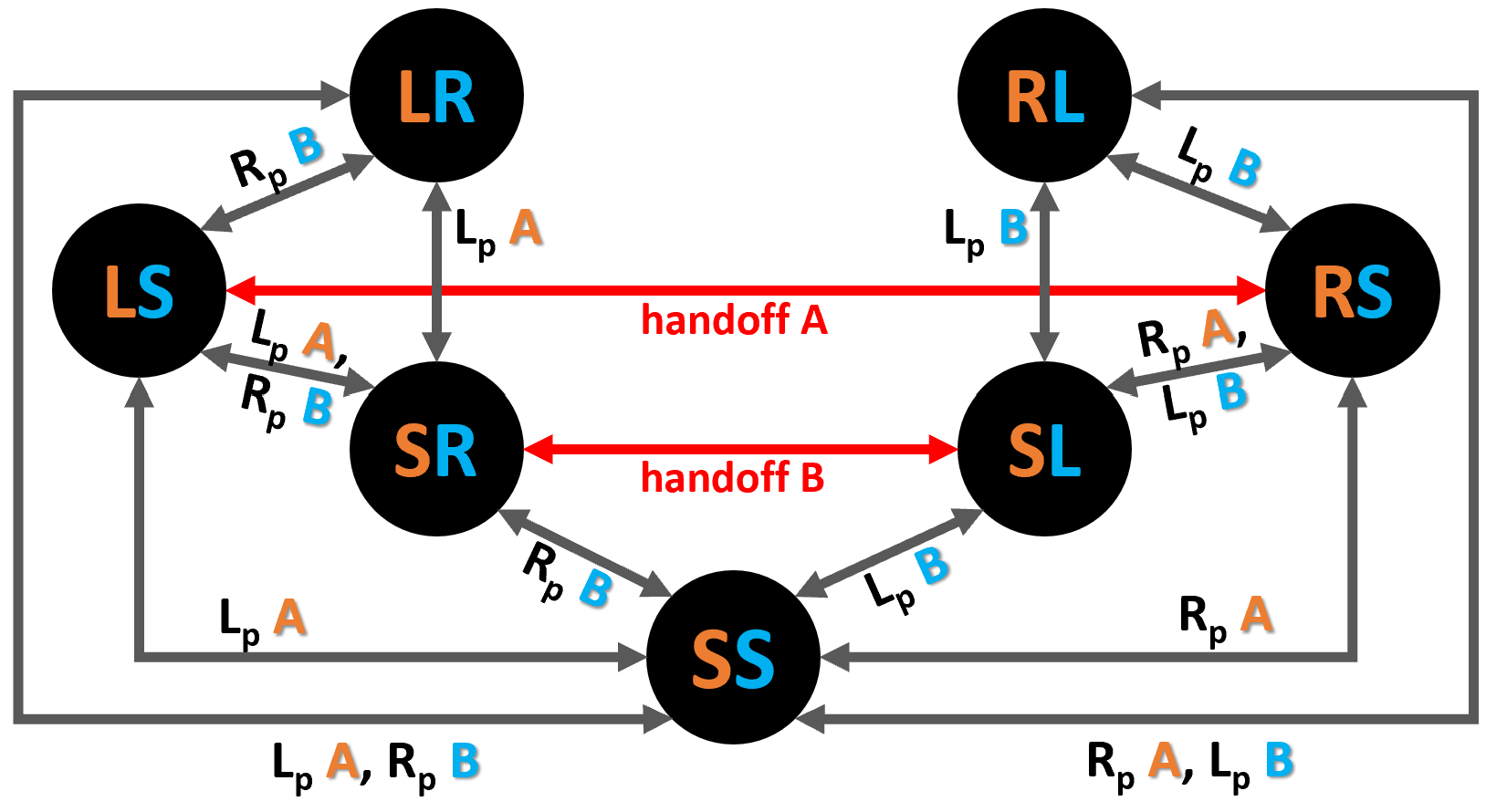}
    \vspace{-0.1in}
    \caption{The image shows the connectivity of different modes for a \moma problem with 2 arms $L$, and $R$ describing pick/place actions $L_p, R_p$ and 2 objects ($A,B$), with object-centric modes (S-stable, L-grasped by left, and R-grasped by right), black pick or placement edges, and red handoff edges.}
        \vspace{-0.3in}
    \label{fig:taskmodes}
\end{wrapfigure}

Let a set of object poses define an arrangement $\arrangement = (\pose_1 \ldots \pose_k)$.  A multi-arm manipulation path $\Pi$ is valid if every task space state along it is collision free for all arms and objects.

\textbf{Rearrangement Problem:} An object rearrangement problem consists of an initial arrangement of objects at $\ainit = (\pose_1^{\rm init} \ldots \pose_k^{\rm init})$ and a target arrangement $\atarget = (\pose_1^{\rm goal} \ldots \pose_k^{\rm goal})$. Each pose in $\ainit$ and $\atarget$ is a stable pose. 
A feasible solution to the rearrangement problem is a valid multi-arm manipulation path that transfers the objects from $\ainit$ to $\atarget$. Along this solution the objects are acted upon by sequences of \textit{picks, places, and handoffs}.

The cost of the solution path $\cost: \prod \rightarrow \mathbb{R}$ is assumed to be the maximum of the Euclidean arc lengths for each $\pi_i$ in $\Pi$, more commonly called makespan. The optimal solution to the multi-arm rearrangement problem $\Pi^*$ is a feasible solution that minimizes the cost.

\subsection{Modeling Choices and Assumptions}
\label{sec:foundations}

A popular framework~\cite{Hauser2011Randomized-Multi-Modal-}\cite{dobson2015planning} for \tamp\ problems corresponds to searching over the space of modes by building a graph that transitions between neighboring modes at the task planning level. This approach makes motion planning calls within each mode so as to identify the transitions. The current work follows a similar framework and builds a tree in the s

\noindent{\bf Conditions on the High-level Task Planner and Solutions Discovered:} In order to study the complexity of the combinatorially large problem of multi-arm rearrangement, the current work considers a simplification of the search-space. We assume that the task planning solution trajectory is decomposed into discrete steps along which the arms can perform \textit{synchronized} execution of manipulation actions, similar to previous work~\cite{shome2018rearrangement}, while allowing a subset of arms to not be involved in any action at a step. Effectively, the set of actions available to each arm is (\textit{pick,place,handoff,NOACT}), where \textit{NOACT} means that the arm has not immediate objective during the current step. This narrows down the search space in two ways: a) each action has a fully defined set of choices (available objects for picks, available placement regions for place, and free arms for handoffs) which makes it sufficient to evaluate these combinations, and b) a sequence of steps is a sequence of actions being assigned to arms. This essentially poses steps as the discretization of time over the task planning solution. Note that during each step all the arms still have to perform centralized coordinated motions. 

\begin{assumption}(Synchronicity)
The task planning solution is decomposed into a sequence of synchronized manipulation actions (picks, placements, handoffs, and NOACT), such that the end of the step is attained only when all the manipulators (except those assigned NOACT) complete their respective actions. It is assumed that a synchronized multi-arm rearrangement (\moma) problem is solvable by such a synchronized solution.
\label{ass:sync}
\end{assumption}

Now that the general problem has been simplified to narrow it down to a search over assignment of manipulation actions to arms over a sequence of steps, the next section goes on to describe how this can be efficiently solved, and then incorporated into a high-level task planner for efficiently guiding solutions.

\subsubsection{Conditions for the Heuristic to be Effective:}

Despite the previous simplification, the number of possible choices available for the assignment of actions to arms and their sequence remains a prohibitively large space to naively search. We introduce additional set of assumptions to bring some structure to the search space, that can describe a subset of the synchronized multi-arm rearrangement problems.

\begin{condition}(Discrete Placement Regions)
Assume that there exists a set of discrete placement regions $$\pregion = \{ \region_1\ldots \region_P \}, \ \region_p \subset SE(3),\  \pose^{\rm place}_j \in \region_p\ \forall \pose^{\rm place}_j.$$

It is assumed that for each such placement region, the set of arms that can reach all the poses contained in the set is the same for every object pose in the region. Additionally, each placement region also has a specific capacity describing the maximum number of objects that can be guaranteed to concurrently lie in it.
\label{ass:discreteplacement}
\end{condition}

\begin{condition}(Uniformity of Object Reachability)
Though there is no explicit restriction on the objects being different, they must possess uniform reachability w.r.t. the decomposition of the discrete placement regions. This means, for every object the set of arms that can reach all poses in a specific placement region are identical. Note that the capacity of the placement region is also assumed to account for the the different object sizes.
\label{ass:objectuniformity}
\end{condition}

This object reachability condition is trivially satisfied if the objects are identical. It should be pointed out that it is desirable to have a minimal number of placement regions, which permits not having to reason about individual poses during the high-level search, thereby reducing the search space. These capacity constrained placement regions can more accurately capture the \textit{discretization} of the workspace (Fig~\ref{fig:moma_mpp}). The sequence of arm actions required for a problem involving a specific combination of placement region should remain unchanged by the specific poses involved. This can promote the reuse of such high-level plans as well.

\begin{condition}(Object Non-interactivity)
For a \textit{pick} action by an arm $\arm$ performed on an object at pose $\pose$, if the pose lies in a reachable placement region, the feasibility of the \textit{pick} action is independent of any other object. For a \textit{place} action by an arm $\arm$ performed on an object to take it to pose $\pose$, if the pose lies in a reachable placement region, the feasibility of the \textit{place} action is only violated by an object currently at a pose intersecting with $\pose$. 
\label{ass:monotone}
\end{condition}

\rahul{Note that object non-interactivity is guaranteed for a tabletop setup with overhead grasps where the start and goal poses of the objects do not intersect.} The condition itself does not preclude non-monotonicity, as demonstrated in Section~\ref{dem:nonmonotone}. Consider a problem that violates Condition~\ref{ass:monotone} and is non-monotone: two objects placed one behind another inside a narrow shelf requiring grasping the deeper object first. Such problems lie outside the efficient subset of \moma problems addressed in this work.

\section{Mode Graph with Capacity Constraints}
\label{sec:modegraph}
Despite the assumptions formulated in the previous section, as the number of arms and objects grows, given the number of available manipulation actions, the assignment of actions to arms, and sequencing them expresses a large search space (shown for only 2 arms and 2 objects in Fig~\ref{fig:taskmodes}). The possible actions available seem to create a notion of connectivity, where available actions are expressed by the setup of the workspace. For instance, if arm $\arm_1$ can reach the placement region $\pregion_1$, but $\arm_2$ cannot, then a pick action on an object in $\pregion_1$ is not a valid action available to $\arm_2$. 
A key insight is that it might be possible to think of the problem in terms of the objects (as in previous work~\cite{cohen2015planning}) as they traverse the connectivity expressed by allowable actions performed upon them. Another key insight is that the connectivity only expresses one of the constraints. There exists a notion of \textit{capacity} for manipulators and placement regions i.e., maximum number of objects involved in single manipulation actions. 

Inspired by the formulation of the multi-agent path finding problem \mpp with vertex capacity constraints~\cite{surynek2019multi}, a contribution of the current work is to pose a object-centric mode graph to express the \moma.

\begin{definition}(Object-Centric Mode Graph)
\rahul{An object-centric mode graph is a directed, weighted graph with vertex capacities. The vertices are either (i) placement regions or (ii) arms; edges map to manipulation actions with the cost of the action encoded in the edge weight; the capacities denote how many objects can occupy a vertex.}
\begingroup
\allowdisplaybreaks
\begin{align*}
    \modegraph{\modenodes}{\modeedges},& \quad \text{has nodes\ } 
    \modenodes = \{ \modes \in \pregion \cup \arms \},\quad \text{and edges\ }
    \modeedges = \{ \actions(\modes_u, \modes_v)\}\\
        \actions(\modes_u, \modes_v) \in& \modeedges \ \text{ if}
    \begin{cases}
    \text{Pick:\ } \modes_u \in \pregion,\ \modes_v \in \arms,\ \text{and }\ \modes_v\ \text{can reach}\ \modes_u\\
    \text{Place:\ } \modes_u \in \arms,\ \modes_v \in \pregion,\ \text{and }\ \modes_u\ \text{can reach}\ \modes_v\\
    \text{Handoff:\ } \modes_u \in \arms,\ \modes_v \in \arms, \modes_u\neq\modes_v,\ \text{and }\ \modes_u\ \text{can reach}\ \modes_v
    \end{cases}\\
    \weights :&\ \modeedges\rightarrow \mathbb{R} \ \text{ are the edge weights}\\
    \capacity :&\ \modenodes\rightarrow \mathbb{W} \ \text{ are the vertex capacities}\\
    \capacity(\modes_u) =&
    \begin{cases}
    1\quad&\text{ if }\quad\modes_u\in\arms\\
    \text{ number of objects that can fit in  }\modes_u\quad&\text{ if }\quad\modes_u\in\pregion
    \end{cases}\\
\end{align*}
\endgroup
\label{def:modegraph}
\vspace{-0.4in}
\end{definition}

\rahul{Each mode $\mode$ in the set of vertices $\modenodes$ corresponds either to placement regions $\pregion$ or arms $\arms$. This means that $card(\modenodes) = card(\pregion) +card(\arms)$.} 

Given $k$ objects, let the set of modes they instantaneously occupy define a \textit{multi-modal state} on the mode graph, $\mstate = (\modes_1,\ldots\modes_k)$.

By definition of the placement regions $\pregion$, a stable arrangement consists of objects lying in a set of nodes of the mode graph $\modeg$. This means,$ \ainit = (\pose_1^{\rm init}\ldots\pose_k^{\rm init}), \quad \pose_j^{\rm init}\in \region_p \in \modenodes\quad \forall \pose_j^{\rm init}$.
Define this mapping as $\modeinv(\ainit) = \mstate^{\rm init}$.
This defines a set $(\pose_1^{\rm init}\ldots\pose_k^{\rm init})$ maps to a set of mode-graph nodes $\mstate^{\rm init} = (\modes_1^{\rm init}\ldots\modes_k^{\rm init})$ for $\object_j$. 

Similarly define  $\mstate^{\rm goal} = (\modes_1^{\rm goal}\ldots\modes_k^{\rm goal})$ for $\atarget$  for each object $\object_j$. With a slight abuse of notation, define an inverse mapping a fully defined task space configuration maps to a set of modes on $\modeg$, such that $ \modeinv(\State) = \mstate$.

\textit{Connection to Task Modes: Compared to Fig~\ref{fig:taskmodes}, $L,R$ map to modes for the arms in $\modeg$, All the placement regions correspond to $S$. For 2 objects, their multi-modal state $\mstate$ lies on one of the nodes in Fig~\ref{fig:taskmodes}. Transitioning to an adjacent $\mstate$ moves both objects, and traverses two edges in $\modeg$ and one edge in the task mode graph.}

By the formulation of feasible multi-arm manipulation solutions, each arm executes a sequence of \textit{moves} and \textit{transfers}, where an arms can manipulate a specific object by \textit{picking}, \textit{placing}, or \textit{handoff}. 
For each object $\object_j$ during the execution of a valid $\Pi$ denote the discrete manipulation actions by the tuples $(\arm_i, t^j)$ for picks, $(\pose_p, t^j)$ for placements, and $(\arm_i, t^j)$ for handoff to $\arm_i$ at instants $\Pi(t^j)$ respectively. 
As an illustrative example, $\Pi$ yields the following sequences of timed manipulation actions from the initial to target poses of each of the $k$ objects, using \textit{picks, placements, and handoffs}. This describes: $\object_j \xrightarrow{\Pi} [ (\pose_j^{\rm init}, 0), (\arm_i, t_j^1), \ldots (\pose_j^{\rm goal}, 1) ]\quad \forall o_j$

Each of the poses in the above sequences correspond to some placement region $\pregion$ by Assumption~\ref{ass:discreteplacement}. Each pick and handoff corresponds to an arm in $\arms$. This means every tuple in the sequences above correspond to some $(\modes, t)$ where $\modes\in\modenodes$ in the mode graph at a specific time parametrization $t$.

\begin{figure}[t]
    \centering
    \vspace{-0.1in}
    \includegraphics[width=0.98\textwidth]{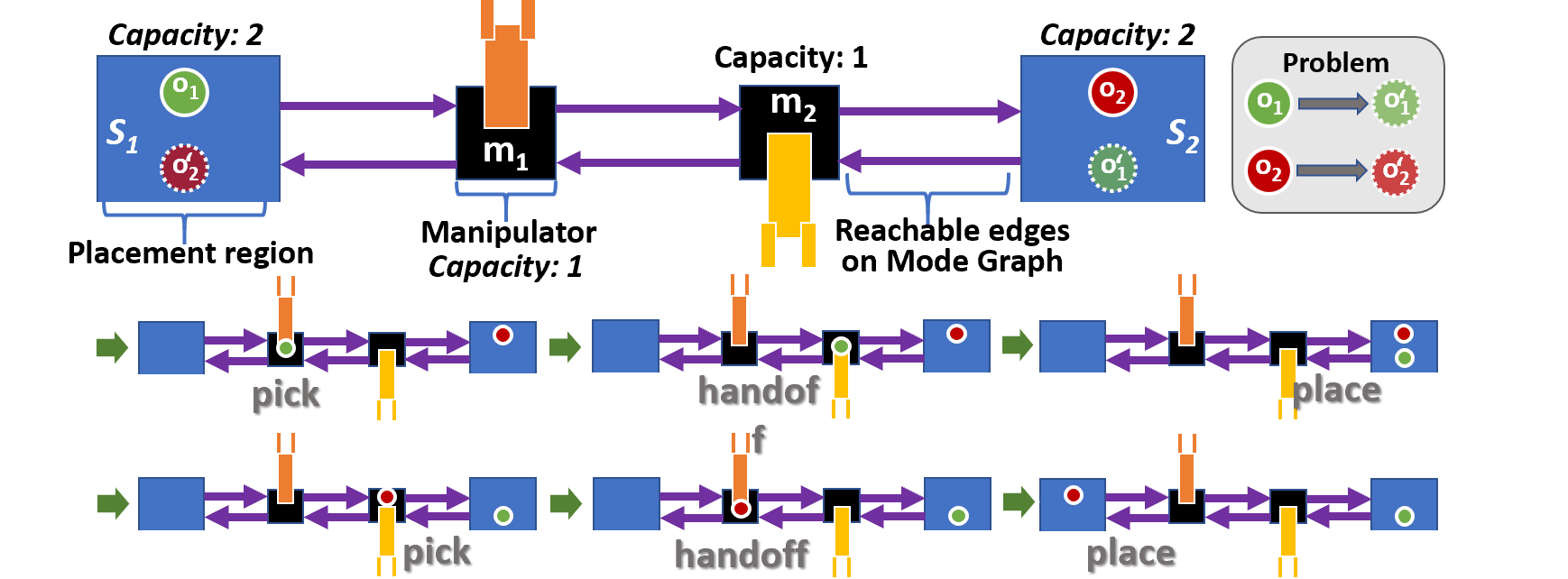}
    \vspace{-0.2in}
    \caption{The figure shows an instance of a multi-arm multi-object rearrangement problem, and the corresponding mode graph $\modeg$, with the \mpp solution obtained for affecting the rearrangement of both the red and green objects across the two placement regions $\pregion_1$ and $\pregion_2$.}
    \vspace{-0.2in}
    \label{fig:moma_mpp}
\end{figure}

\begin{definition}(Multi-agent Path Finding (\mpp))
\rahul{Given a starting configuration $\mstate^{\rm init}$, and a desired final configuration is $\mstate^{\rm goal}$ on the mode graph $\modeg$, a multi-agent path finding solution finds a sequence of \textit{valid} vertices and edges for each object that takes them from their start to goal vertices.} 
\label{def:mpp}
\vspace{-0.1in}
\end{definition}

A solution to the \mpp problem defines for every object a sequence of nodes on the graph along with the time at which the object starts occupying the node, such that each object begins at the initial node, and ends at the target node.
Let $\Pi_\objectset$ be a feasible solution to the \mpp problem of $\objectset$ agents on $\modeg$, consisting of a sequence of multi-modal states $\Pi_\objectset=[\mstate^{\rm init},\ldots\mstate^{\rm goal}]$. In terms of each object:
\vspace{-0.08in}
\begin{align*}
    \Pi_\objectset = 
    \begin{cases}
    \pi_{\object_j} : [(\modes^j_{\rm init}, 0),\ldots (\modes^j_l, t)\ldots (\modes^j_{\rm goal}, 1)] \quad \forall \pi_{\object_j}
    \end{cases}
\vspace{-0.1in}
\end{align*}
\textbf{Constraints on motions:}
Certain constraints on the allowable motions of the objects over the mode graph $\modeg$ respect the restrictions of the problem. The following two are derived from the classical formulation of \mpp: (i) at any time $t$, no vertex $\modes$ exceeds its capacity $\capacity(\modes)$, and (ii) agents do not cross each other over the same edge 
Vertex capacities in $\modeg$ is an additional constraint that needs to be imposed to make the problem amenable to multi-arm rearrangement. \textit{Vertex capacities cap the maximum number of objects that can move over all the in-edges and out-edges of every vertex at any step}.

\noindent\textbf{Cost of \mpp:} Each edge on the graph has a weight defined by $\weights$. For each object, the solution cost is $\cost(\pi_{\object_j}) = \sum \weights(\actions_t),\ \forall \ \actions_t \text{ in the solution}$.
The cost of the solution $\cost(\Pi_\objectset)$ is typically defined as the maximum of each component $\cost(\pi_{\object_j})$. 

\noindent\textbf{Optimal \mpp:} An optimal \mpp solution, $\Pi_\objectset^*$ is a feasible \mpp solution that also minimizes the cost of the \mpp solution.

\begin{theorem}(Multi-arm Rearrangement Solution Solves \mpp on $\modeg$)
Each feasible solution to the multi-arm multi-object rearrangement problem $\Pi$ for $r$ arms and $k$ objects from arrangement $\ainit=(\pose_1^{\rm init}\ldots\pose_k^{\rm init})$ to $\atarget=(\pose_1^{\rm goal}\ldots\pose_k^{\rm goal})$ reduces to a feasible solution to a multi-body path planning problem of $k$ agents on $\modeg$ defined in Def~\ref{def:modegraph} from the corresponding $\mstate^{\rm init}$ to $\mstate^{\rm goal}.$
\end{theorem}
\begin{proof}
\rahul{By construction, there is a one-to-one correspondence between the manipulation actions necessary for the rearrangement problem and edges in $\modeg$}. Solutions in the complex task planning space map to the solution to the \mpp problem on a graph, where each object starts moving at the timestamp corresponding to action involving the object in the task planning solution $\Pi$, before progressing to the next vertex on $\modeg$. At the end of the solution, every object ends up in the target vertices on $\modeg$. \qed

\end{proof}

\begin{definition}(Discrete \mpp on $\modeg$)
In this problem, starting from an initial set of vertices $\mstate^{\rm init}$ in $\modeg$, at each \textit{discrete} step $t$ each object can move to an adjacent vertex or stay in place. Each edge traversal incurs a cost corresponding to the edge weight.  
After tracing the solution steps each object reaches the final configuration $\mstate^{\rm goal}$.
\end{definition}

\textit{Note that in the discrete version, every time step and edge traversal is atomic.}
This aligns with the synchronicity assumption of \moma as well.
The number of edges in the solution, per object, represents the number of \textit{picks, places or handoffs} involved. For each object the cost of the solution being the number of actions is analogous to counting the number of edges in the object's solution, i.e., $ \text{if}\ \weights: \modeedges \rightarrow 1, \ \cost(\pi_{\object_j}) = \text{\#actions} $.

\begin{theorem}(An action optimal \mpp solution is an admissible heuristic for Synchronized Multi-arm Rearrangement)
\rahul{The number of steps in the \mpp solution over $\modeg$ is less than or equal to the number of synchronized actions (\textit{picks, places, and handoffs}) in the corresponding synchronized multi-arm rearrangement solution.}
\end{theorem}

\begin{proof}
The cost of the solution is $\cost(\Pi_\objectset) = \max (\cost(\pi_{\object_j}),\ldots \cost(\pi_{\object_j}))$. Using this cost, the optimal solution $\Pi_\objectset^*$ minimizes the maximum number of actions on any object. 
This analogy is applicable under the assumption set up in Section~\ref{sec:foundations}. The synchronicity assumption allows for these synchronized actions that reflect the atomic multi-agent steps over the mode graph. Additionally, the object non-interactivity, and reachability assumptions ensure that the edges on the mode graph are effectively independent of each other and reflect \textit{feasible} and \textit{reachable} actions for the arms. \qed
\end{proof}

In general the feasibility of \rahul{high-dimensional motion planning for} each manipulation action on the \mpp solution is not guaranteed. The \mpp solution can still be useful, however, as a suggested sequence of actions that can be checked first for feasibility.

\begin{observation}(A \mpp solution on $\modeg$ describes a sequence of actions for each arm)
\rahul{A \mpp solution to $\mstate^{\rm goal}$ over the mode graph corresponds to a sequence of manipulation actions, which if collision-free will bring the objects to the poses in $\mstate^{\rm goal}$}.
\label{thm:heuristic}
\end{observation}

A configuration of arms and objects $\State^{\rm current}\in\taskspace$ describes the corresponding multi-modal state $\mstate^{\rm current} = \modeinv(\State)$. Define the \mpp problem from $\mstate^{\rm current}$ to $\mstate^{\rm goal}$. 
\rahul{
The \mpp solution $\Pi_\objectset = (\mstate^{\rm current},\mstate^{\rm next},\ldots\mstate_l\ldots\mstate^{\rm goal})$ describes a sequence of vertices on $\modeg$ and corresponding edges or actions(\textit{picks, places and handoffs}) for each object and arm. 
}
\vspace{-0.1in}
\begin{align*}
    \heuristic = \Big(   \actions_1, \actions_i,\ldots \actions_r   \Big), 
    \actions_i = \begin{cases}
        \emptyset,\quad& \text{ if arm }\ \arm_i\ \text{not in solution }\ \Pi_\objectset\\
        \actions,\quad& \text{ the first action involving arm } \arm_i\ \text{in } \Pi_\objectset
                    \end{cases}
\end{align*}
\vspace{-0.05in}

A slightly different estimate is used as the multi-modal goal $ \modegoal = \Big(   \actions_1, \actions_i,\ldots \actions_r   \Big)$, 
\rahul{which comprises of an action per arm required next, i.e., $  \actions_i = \emptyset$ if arm $\ \arm_i$ not used to transition to $\mstate^{\rm next}$. Otherwise the corresponding the action involving the arm $i$ in $\mstate^{\rm next}$ is used. Thus, we obtain a longer horizon heuristic and an immediate biasing goal for defining the next action of each arm to trace the \mpp solution towards $\mstate^{\rm goal}$.}

Note that while $\modegoal$ makes stepwise progress along the \mpp solution, the heuristic expresses something more powerful. By biasing towards the heuristic, each arm can \textit{pre-empt} what is required of it next, even if it is farther into the future than one single step of the \mpp solution.

\section{Integer Linear Program for MAPF on Mode Graph}
\label{sec:ilp}

This section outlines the solution to the \mpp problem over the mode graph. An integer linear programming model is set up on the lines of previous work~\cite{yu2016optimal}. At a high-level, the method takes as input the mode graph, and constructs a \textit{time-expanded} graph where a) nodes are replicated across time slices, and b) each pair of bidirectional (or undirected) edge is replaced with a gadget 
(Fig~\ref{fig:gadget})  
across two time slices. 
\begin{wrapfigure}{r}{0.3\textwidth}
	\centering
	\vspace{-0.3in}
	\includegraphics[width=0.29\textwidth]{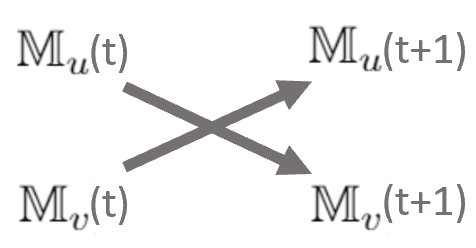}
	\vspace{-0.1in}
	\caption {
	Gadget for constructing the time-expanded graph from the mode graph.
	}
	\vspace{-0.1in}
	\label{fig:gadget}
\end{wrapfigure}
By connecting the start and goal nodes between the first and last time slices, the solution to the multi-commodity flow problem over the time-expanded graph ($\modeg^T$) describes an \mpp solution.

For the sake of brevity most of the replicated details are omitted here. The specific addition required to guarantee solutions that are usable over the mode graph, formulates the capacity constraints at the modes. Inspired by previous work~\cite{surynek2019multi} the notion of capacity has to be modeled correctly. These capacity constraints essentially encode the restriction that \textit{through} any instant step (or time slice) at capped number of objects can interact with the mode.

Let $v$ denote a node in the time expanded graph corresponds to some mode $\modes_v$, and $e_j$ denote each edge. There are indicator variables assigned to each robot ($i\in[1\ldots k]$) moving across an edge $e_j$ as $x_{i,j}$. Let $\delta^+(v)$ be the out-edges and $\delta^-(v)$ be out-edges.

\vspace{0.1in}
\textbf{Capacity Constraint: } $\underset{{e_j\in\delta^+(v)}}{\sum} x_{i,j} + \underset{{e_j\in\delta^-(v)}}{\sum} x_{i,j} \leq \capacity(\modes_v)\  \forall v \in \modeg^T\ , 1\leq i \leq k $
\vspace{0.1in}

Minimizing the maximum traveled distance with uniform modal edge costs  minimizes the maximum number instantaneous actions in the \mpp solution over $\modeg$.

\noindent {\bf Implementation Details}:
The current objective is to use an underlying \mpp solver as a quick heuristic over the mode graph $\modeg$. It is therefore beneficial to encode some measure of the cost of the actions represented by edges in the mode graph in \mpp.

\textbf{Costs:} The cost of an edge in the mode graph provided by the $\weights$ function is assumed to be the position distance in $\mathbb{R}^3$ of the a) centroid of the placement region poses, b) the positions of the root frames for each arm. 
The objective remains minimizing the maximum traveled distance on $\modeg^T$.

\textbf{ILP Invocation:} In previous work~\cite{yu2016optimal} the underlying ILP model was augmented for a range of time-steps ($T_{\rm min}$ to $T_{\rm max}$) and called repeatedly. 
The current work needs the solution to be obtained ideally in a single ILP call.
We restrict the number of ILP invocations by the following change:
\vspace{-0.1in}
\begin{align*}
\vspace{-0.2in}
    T_{max} =  r\times k. \quad \quad 
    \cost'(e_{i,j}) = time(e_{i,j})\cdot \cost(e_{i,j})
    \vspace{-0.2in}
\end{align*}
where $ time(e_{i,j}) $ returns the time slice corresponding to the source of $e_{i,j}$. Using the reformulated cost $\cost'$, in combination with the weighted edges, we might lose some of the strict optimality guarantees, but this lets us get solutions which a) reason about heuristic costs, and b) try to minimize the number of time steps the robots move (i.e., prefer to lower maximum number of actions).

The solution can be retrieved by pruning the end of the solution where the objects reach their target modes and stay still for the remainder of the time-steps. The results from experiments indicate that the modified \mpp solver remains useful as a heuristic.

 
\section{Integrated Task and Motion Planning}
\label{sec:tamp}
An integrated task and motion planning approach has to simultaneously explore (i) arm configurations $\cspace_1\times\ldots\cspace_r$ for every arm, and (ii) \textit{picks, places, and handoffs} that change object poses for every object.

The underlying search is similar to previous work~\cite{shome2019anytime}, which described a multi-modal integrated task and motion planning algorithm using an underlying $\drrtstar$\cite{shome2019drrt}-like decomposition. It has been shown in previous work~\cite{shome2019drrt,shome2018rearrangement,shome2019anytime} that this search strategy is critical for solving these high-dimensional multi-robot problems. The key differences in the proposed approach are outlined below.\\
        a) the integratation the sampling of transition configurations that achieve neighboring modes inside the online search process. For instance, if a neighboring mode on $\modeg$ involves the arm $\arm_i$ picking up $\object_j$, an IK solver can be invoked to find a set of grasping configurations to add as possible grounded configurations to plan to.\\
b) A multi-body path planning heuristic over the mode-graph is used to bias the high-level search, that determines what arm actions would be preferred. In a lot of problem instances, the search space is so large, and the multi-body coordination so constraining, this heuristic proves crucial to not only solution quality, but also feasibility within a limited time budget. Next we  describe some underlying modules that are assumed to be available to the algorithm.

\textbf{Mode-graph Generator:} There is a module which is aware of the spatial configuration of all the robots and placement regions, and their respective reachability, in order to generate a set of placement regions satisfying Assumption~\ref{ass:objectuniformity}, returning $\modeg$. 

\textbf{\mpp Solver:} An \mpp solver subroutine is assumed to be available (as is described in Section~\ref{sec:ilp}. Given an initial multi-modal state, the module returns the multi-body path planning solution for the objects over the mode graph with capacity constraints, that ideally minimizes the maximum path cost by of object.

\textbf{Transition Sampler:} It is assumed that a transition sampler can generate complete configurations defining manipulation actions like picks and placements involving an object and an arm, or handoffs involving two arms and an object. The sampler should be able to provide any number of these mode transitions, whenever invoked.

\subsection{Algorithm}

Bringing together all the tools described so far, a forward search tree to solve sychronized multi-arm rearrangement (\momaalgo) is outlined in this section. The method constructs a search tree $\tree$ of task space configurations $\State\in\taskspace$, rooted at $\State_{\rm init}$ with the objects at $\ainit$, the objective being to find a continuous sequence of motions that reach a state such that the objects are in their final arrangement $\atarget$.

Each vertex of the tree consists of $\State$ and an unique identifier $\tau$ keeping track of the sequence of manipulation actions that led to the current tree vertex. This identifier is similar to \textit{orbits} described in previous work~\cite{vega2016asymptotically}.
The high-level algorithm is described in Algo~\ref{algo:moma}. The method first builds a mode graph calling the subroutine $\mathtt{build\_mode\_graph}$. As mentioned, an internal counter must keep track of different sequences of actions ($\tau_{\rm count}$). 

During each iteration, a $\mathtt{select\_mode}$ selects a mode $\modes$ and the transition id $\tau$ from the tree. 
A fraction of the invocations, the subroutine is designed to \textit{goal bias} by selecting modes that have made the most progress towards the goal. 
If the selected $\tau$ has never been expanded before, the heuristic function for it would be empty. In such a case, the multi-body motion planning subroutine $\mathtt{\mpp}$ is called to obtain a sequence of usable actions ($\heuristic$) that can guide the arms. The computation of this heuristic follows Theorem~\ref{thm:heuristic}. 
If not invoked before, or a fraction of the expansions, neighboring grounded transition configurations are added. This is tantamount to inspecting the current multimodal state and expanding the set of available \textit{grasping, placement or handoff} configurations.
\begin{wrapfigure}{r}{0.61\textwidth}
\vspace{-0.3in}
\begin{minipage}{0.61\textwidth}
    \begin{algorithm}[H]
  \caption{{\tt \momaalgo}$ (\State_{\rm init}, \atarget) $}
  \label{algo:moma}
  $\Pi\leftarrow\emptyset;\quad  \tau_{count} \leftarrow 0$;\\
  $\modeg \leftarrow \mathtt{ build\_mode\_graph}()$;\\
  $\tree.V  \leftarrow <\State_{\rm init}, \mathtt{increment}(t_{count})>$;\\
  \For{$ max\_iters $}
  {
    $(\modes,\tau)\leftarrow \mathtt{select\_mode}()$;\\
    \If{$\heuristic(\tau) = \emptyset$} 
    {
        $\heuristic(\tau) \leftarrow \mathtt{\mpp}(\modes,\modeinv(\atarget))$;\\
    }
    \If{$\heuristic(\tau) \neq \emptyset$ \textbf{or} $\mathtt{add\_fraction}()$} 
    {
        $\mathtt{sample\_trans}(\mathtt{adj}(\modeg, \mode), t)$;\\
    }

    $\State_{\rm near} \leftarrow \mathtt{select}(\tau,\heuristic)$;\\
    $\State_{\rm new} \leftarrow \mathtt{extend}(\State_{\rm near},\tau,\heuristic)$;\\
    $\State_{\rm best} \leftarrow \mathtt{rewire}(\State_{\rm new})$;\\
    
    \If{ $  \Pi(\State_{\rm best}\rightarrow\State_{\rm new}) \in \tfree $  }
    {
        $\tree.V \leftarrow \tree.V \cup <\State_{\rm new}, \tau>$;\\
        $\tree.E \leftarrow \tree.E \cup $\\$ (<\State_{\rm best}, \tau>, <\State_{\rm new}, \tau>)$;\\
        
        \If{$\mathtt{modal\_check}(\modeinv(\State_{\rm new}),\tau)$}
        {
            $\tau_{\rm new}\leftarrow \mathtt{increment}(\tau_{count})$;\\
            $\tree.V \leftarrow \tree.V \cup <\State_{\rm new}, \tau_{\rm new}>$;\\
            $\tree.E \leftarrow \tree.E \cup $\\$(<\State_{\rm new}, \tau>, <\State_{\rm new}, \tau_{\rm new}>)$;\\
        }
        \If{$\modeinv(\State_{\rm new}) = \modeinv(\atarget)$}
        {
            $\Pi_{new} \leftarrow \mathtt{retrace}(\State_{\rm new})$;\\
            \If{$\cost(\Pi_{new}) < \cost(\Pi)$}
            {
                $\Pi\leftarrow\Pi_{new}$;\\
            }
        }
    }
  }
  $\mathbf{return}\ \Pi$
  \end{algorithm}
  \end{minipage}
  \vspace{-0.3in}
  \end{wrapfigure} The subroutines $\mathtt{select}$,$\mathtt{extend}$ and $\mathtt{rewire}$ follow standard definitions except the restriction that they operate on the set of vertices that have the same id $\tau$. The heuristic is used in $\mathtt{select}$ and $\mathtt{extend}$ to bias towards making progress towards various modal goals. Specific consideration has to be provided for modal guidance that is inherently coupled (handoffs) versus single arm targets (picks or places). 
If the extension is valid and collision-free, and the new node satisfies a transition to an adjacent mode, the search can progress to an adjacent multimodal state through the current transition sequence (tracked by incrementing $\tau$). If the adjacent mode attains the target arrangement, the solution path $\Pi$ is updated if the cost improves.\\
\textit{Note on synchronicity:} In the algorithm, the synchronicity assumption is encoded into the function $\mathtt{modal\_check}$ which checks against the multimodal goals like $\modegoal$ returned by the \mpp solver or fully defined multimodal states defined by the  $\mathtt{sample\_trans}$ subroutine. This restricts the search from adding new \textit{transition id-ed} components to $\tree$ every time partial progress has been made, and enforces the synchronization assumption.

\section{Results}
\label{sec:results}
\vspace{-0.15in}
This section goes over the experimental evaluations performed to demonstrate the effectiveness of the proposed approach. The same underlying task planning framework is used in each experiment. The key metrics we want to measure are \textit{the time it takes to find the initial solution}, \textit{the solution cost returned after 30s}, \textit{the number of actions}, and \textit{success ratio}. All experiments were run on a single core of an { Intel(R) Xeon(R) CPU E5-1660 v3 @ 3.00GHz} processor having a maximum available {16GB of RAM}, and data is reported averaged over 20 trials.

{\bf Comparison Points: }
Three strategies are used as guidance in the task planner.

\noindent\textbf{SMART} (proposed): Algo~\ref{algo:moma} uses the \mpp over the mode graph with capacity constraints to the goal arrangement of objects in the mode graph, and uses this solution as the guidance from the current mode.

\noindent\textbf{Sequential:} The objects are prioritized in an arbitrary way that is fixed per experiment. The \textit{Sequential} heuristic solves the \mpp problem over the mode graph \textit{for the next remaining object}. 
This chains a sequence of single-object solutions.

\noindent\textbf{Greedy:} The \textit{Greedy} heuristic solves the \mpp problem for every object, and guides towards the next mode for \textit{every} object greedily. Any conflicts that might arise from this guidance have to be overcome by the exploration in the underlying task planner.

\begin{figure}
    \centering
    \vspace{-0.2in}
    \includegraphics[height=0.9in]{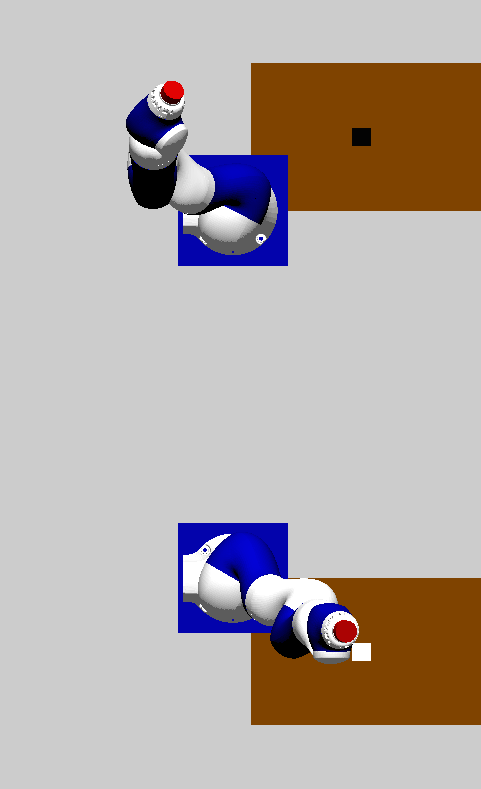}
    \includegraphics[height=0.9in]{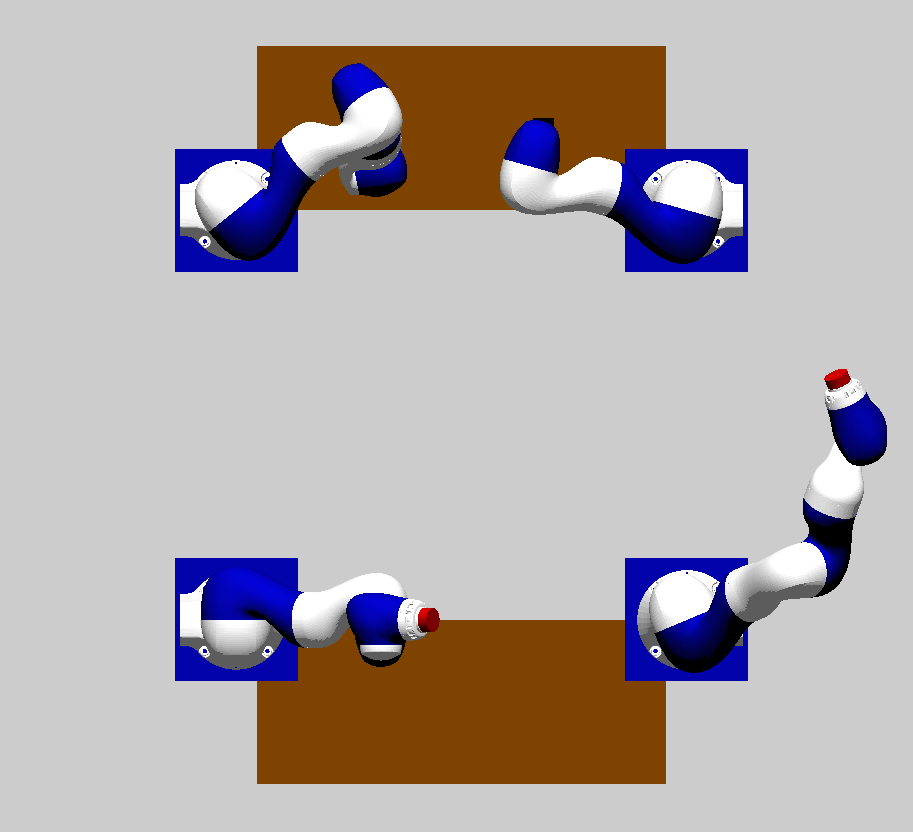}
    \includegraphics[height=0.9in]{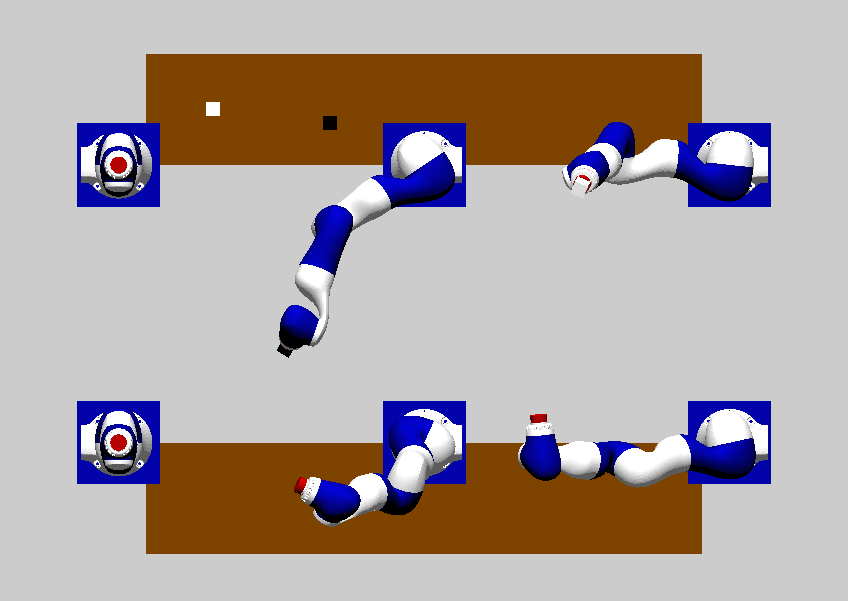}
    \includegraphics[height=0.9in]{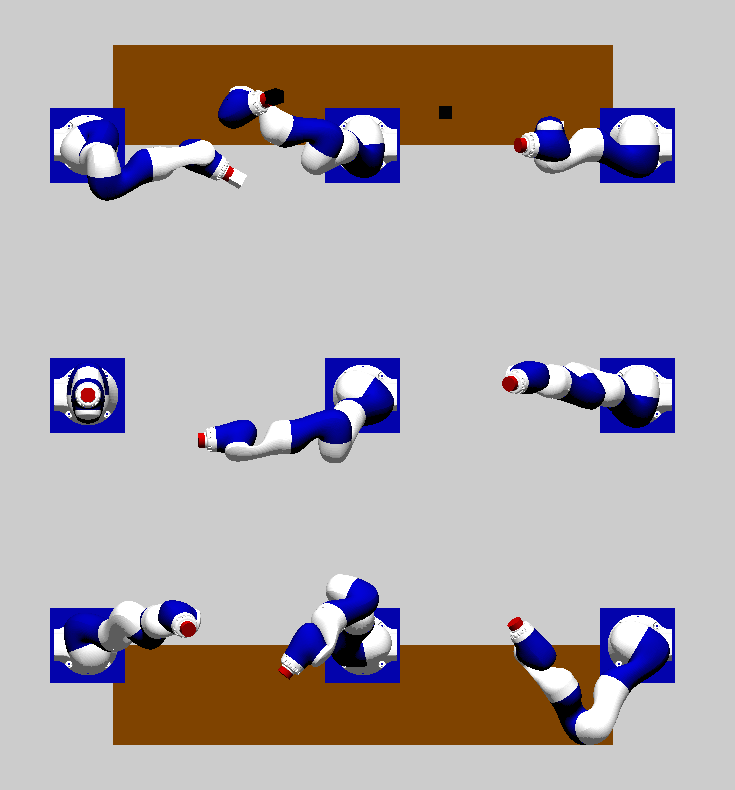}
    \vspace{-0.1in}
    \caption{Benchmarks setups with (\textit{left to right}) two, four, six, and nine arms respectively.}
    \vspace{-0.25in}
    \label{fig:benchmark_setups}
\end{figure}

{\bf Benchmarks: }
The experiments inspect the problem of transferring objects between two tables at two ends of the workspace using a set of $7 \dof$ \textit{Kuka iiwa14} manipulator. 
Fig~\ref{fig:benchmark_setups} outlines the combinations of $r\times k$ as the number of robots and objects used in four different such setups: $2\times 2, 4\times 4, 6\times 4, 9\times 4$.  Note that simply the centralized motion planning problem for $9$ robots lies in a $63$-dimensional space. 

To simplify manipulation the objects are $6cm\times6cm$ cubes with a grasp on each side, and $5$ IK solutions for each (\textit{pick, place, or handoff}), whose computation is included in the reported times. Three of the benchmarks have non-intersecting starts and goals.

\noindent\textbf{Switch:} In this benchmark objects are split into two halves. Each set has initial poses sampled on one of the tables and final poses on the other one. For the other half the transfer direction is reversed. Given the setup, this can cause bottlenecks for \textit{Greedy} which can get stuck if the problem involves poses on opposite placement regions. 

\noindent\textbf{Side-to-side:} In this benchmark all the objects have their initial poses sampled on one of the tables, while the final poses are sampled on the other table. This promotes concurrent transfers using all the arms adjacent to the tables.  

\noindent\textbf{Random:} In this benchmark, for every object the start pose is selected on either table, and the target lies on the opposite table. This in essence is a combination of the other two benchmark with the direction of the transfers randomized.

\vspace{-0.2in}
\subsection{Switch Benchmark}
\vspace{-0.1in}
The data as shown in Fig~\ref{fig:data}(\textit{top}). 
\textit{Greedy} always gets bottlenecked in $2\times 2$. With more manipulators, the \textit{Greedy} success stays low. Note that both \textit{\mgcc} and \textit{Sequential} have similar performance for $2\times 2$ since the solutions should be very similar. As the number of robots increases \textit{\mgcc} the initial times are either comparable or better than \textit{Sequential}, while providing better solutions, and far fewer manipulation actions (since \mgcc minimizes this).  The improvement in the number of actions is not identical to the solution durations since there is an overhead of arm coordination that might exist since {\mgcc} allows an arbitrary number of arms to interact at any time.

\begin{figure}[t]
    \centering
     \includegraphics[width=0.4\textwidth]{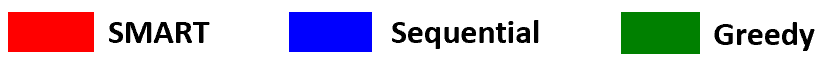}\\
    \includegraphics[width=0.23\textwidth]{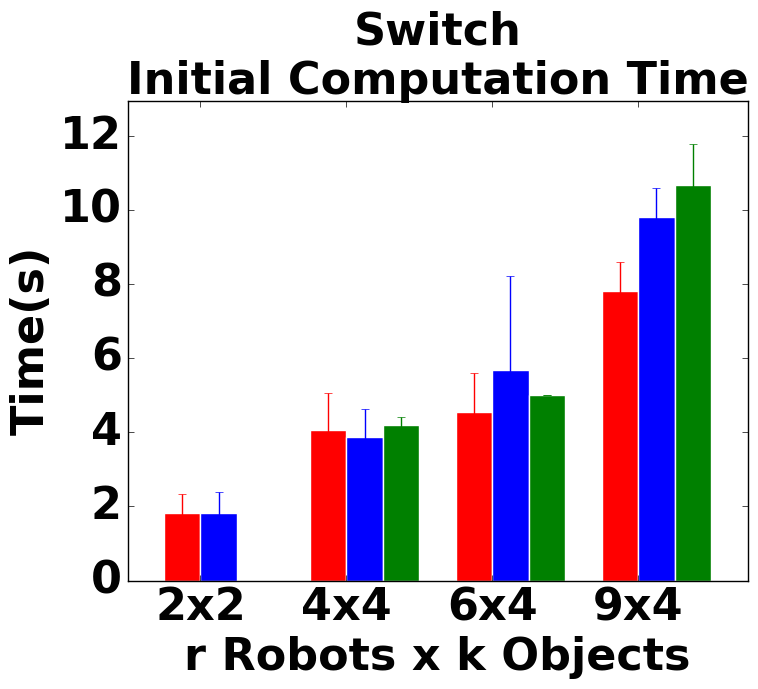}
    \includegraphics[width=0.23\textwidth]{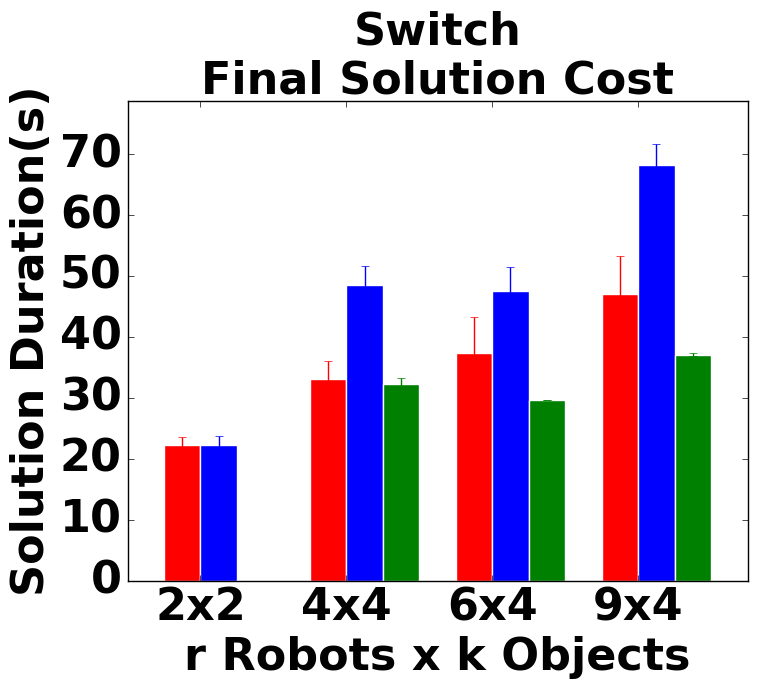}
    \includegraphics[width=0.23\textwidth]{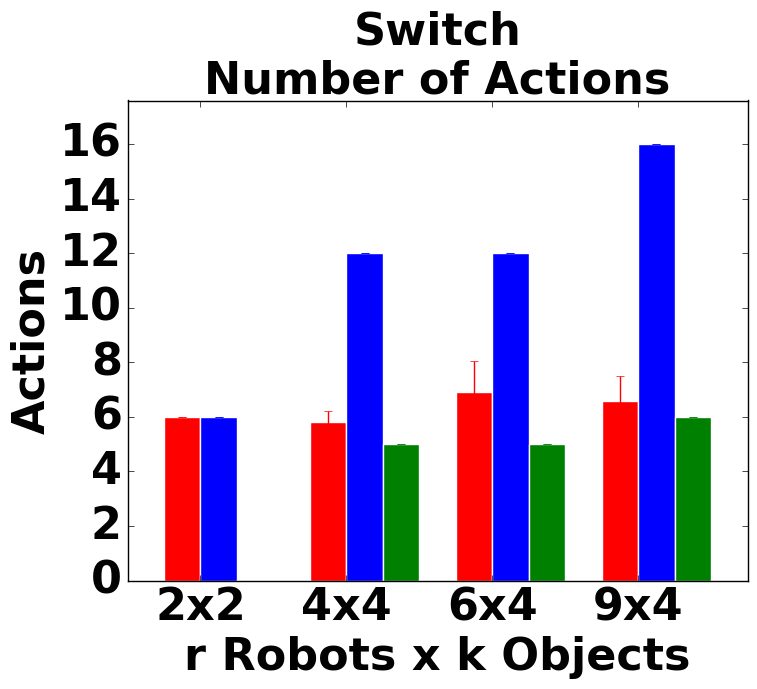}
    \includegraphics[width=0.23\textwidth]{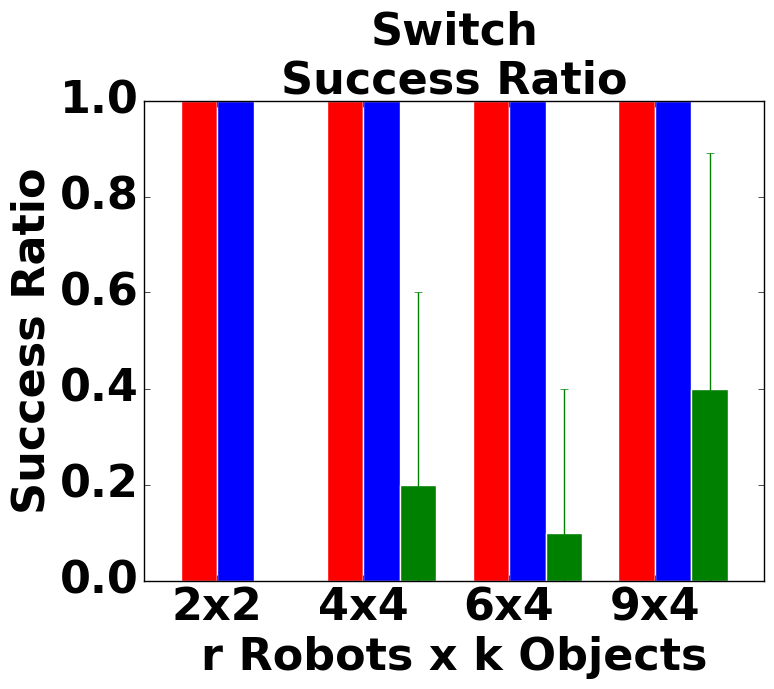}
    \includegraphics[width=0.23\textwidth]{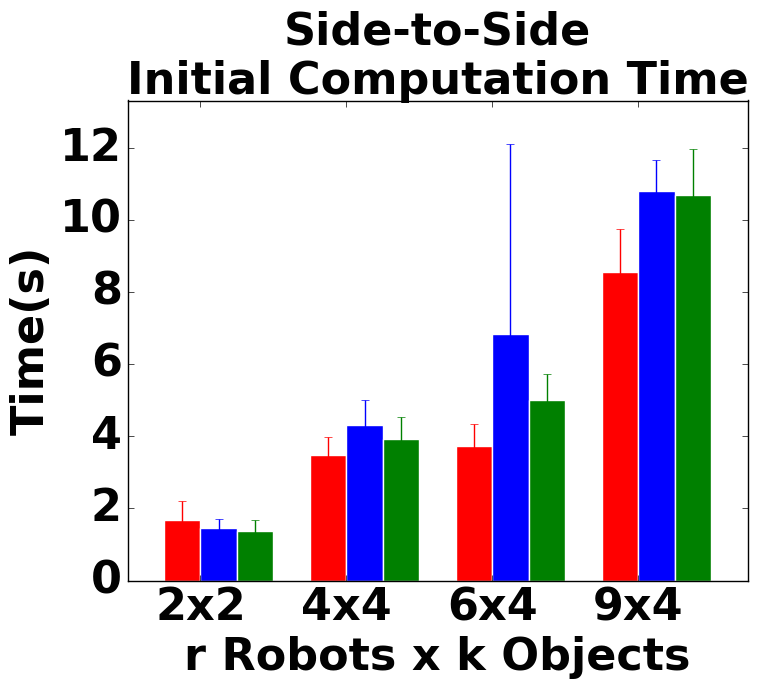}
    \includegraphics[width=0.23\textwidth]{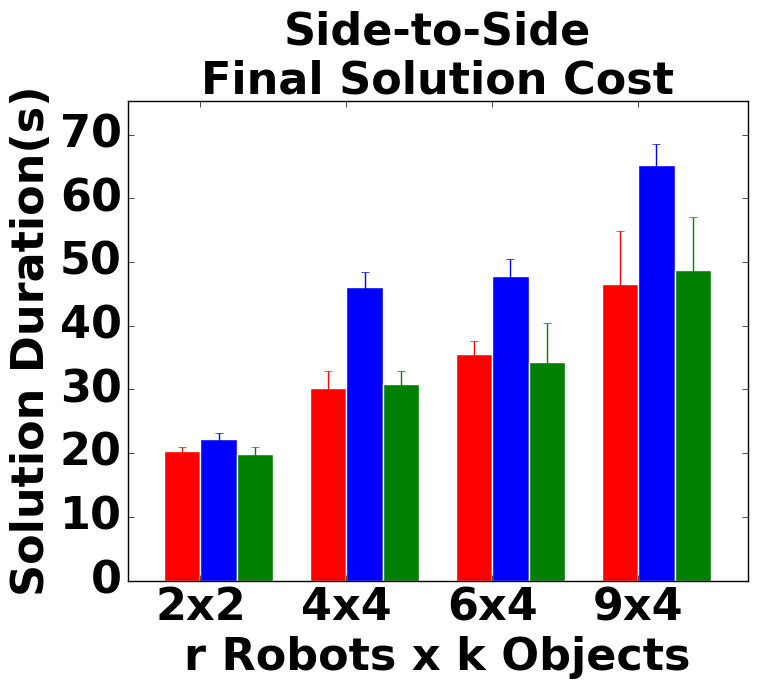}
    \includegraphics[width=0.23\textwidth]{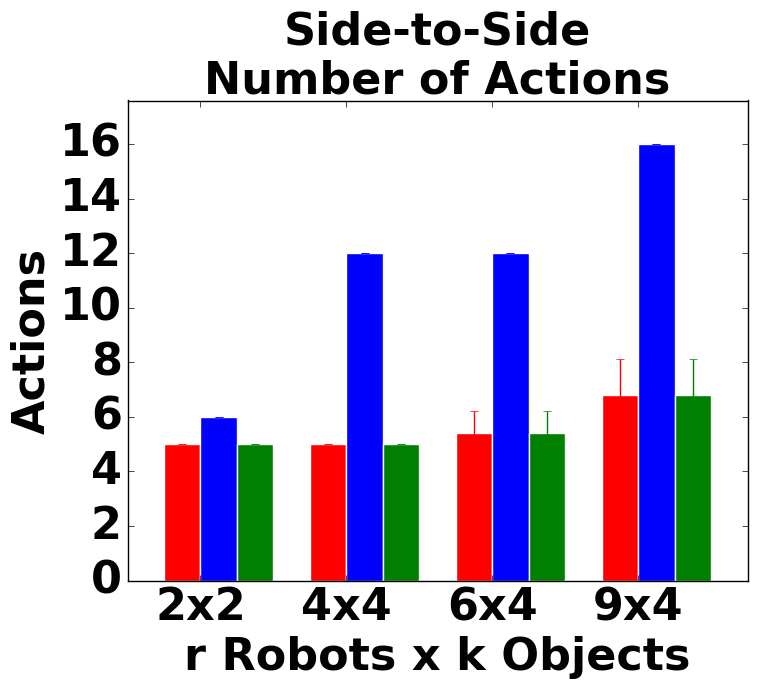}
    \includegraphics[width=0.23\textwidth]{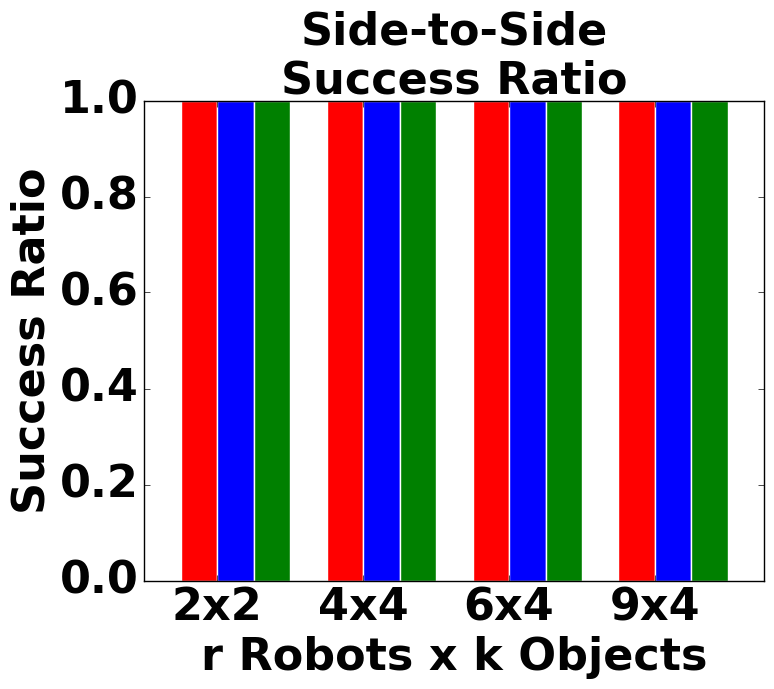}
    \includegraphics[width=0.23\textwidth]{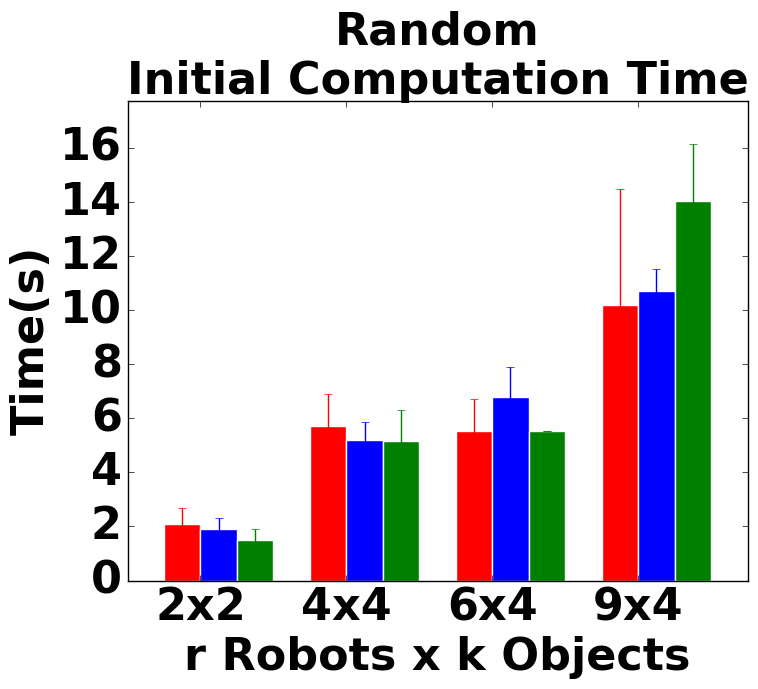}
    \includegraphics[width=0.23\textwidth]{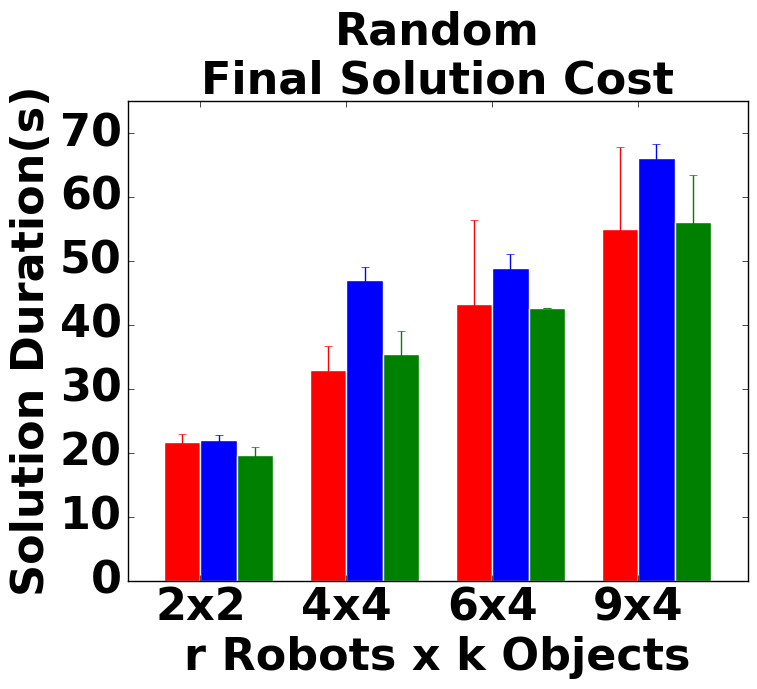}
    \includegraphics[width=0.23\textwidth]{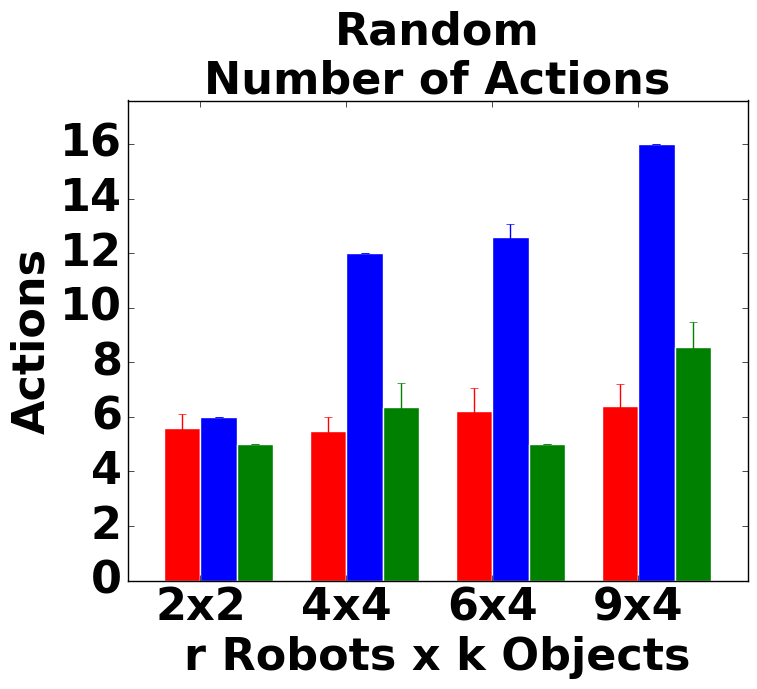}
    \includegraphics[width=0.23\textwidth]{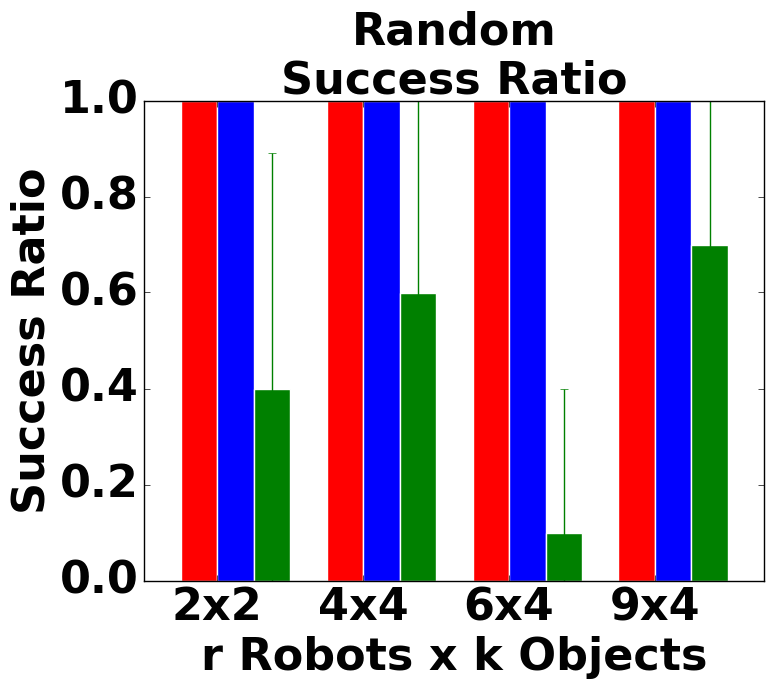}
    \vspace{-0.1in}
    \caption{Benchmarks per row from top to bottom, the switch, side-to-side, and sort benchmark data is reported.  From left to right, (\textit{First:}) shows the initial computation times, (\textit{Second:}) shows the solution duration after 30s of computation, (\textit{Third:}) shows the total number of discrete instants of object transitions or actions in the solution, and (\textit{Fourth:}) the success ratio.}
    \vspace{-0.2in}
    \label{fig:data}
\end{figure}

\vspace{-0.2in}
\subsection{Side-to-side Benchmark}
\vspace{-0.1in}
The data as shown in Fig~\ref{fig:data}(\textit{middle}). All three comparison points succeed in this problem. \textit{\mgcc} needs far fewer actions and a faster initial solution time. The costs are comparable for \textit{\mgcc} and \textit{Greedy}, while being better than \textit{Sequential} solutions. For instance in the $9\times4$ case \textit{\mgcc} has a $20s$ speedup compared to \textit{Sequential}.

\vspace{-0.2in}
\begin{wrapfigure}{r}{0.31\textwidth}
\vspace{-0.35in}
    \centering
    \includegraphics[width=0.31\textwidth]{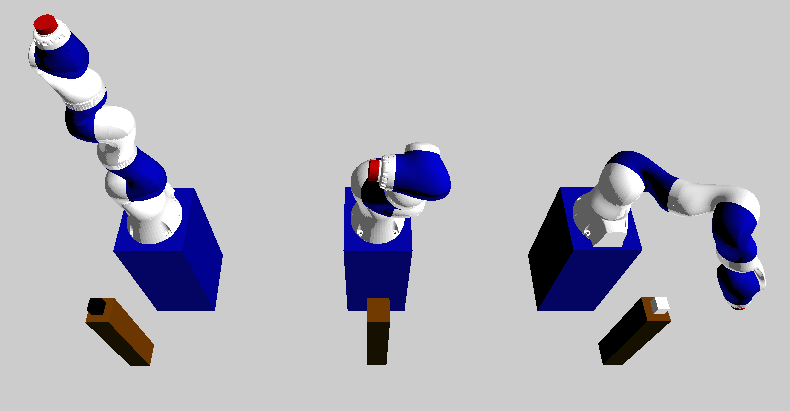}
            \vspace{-0.3in}
    \caption{A non-monotone in-place swap demonstration.}
    \vspace{-0.4in}
    \label{fig:nonmonotone}
\end{wrapfigure}
\subsection{Random Benchmark}
\vspace{-0.1in}
The data is shown in Fig~\ref{fig:data}(\textit{bottom}). In terms of computation times \textit{\mgcc} is similar or better to \textit{Sequential}.  \textit{Greedy} shows poor performance with bottlenecks and fails often. Its solutions are similar to \textit{\mgcc} when it works. As the robots' number increases, its time increases but \textit{\mgcc} manages to solve most problems under $10s$ with fewer manipulation action instants and better solution costs.

\vspace{-0.2in}
\subsection{Non-monotone Demonstration}
\vspace{-0.1in}
\label{dem:nonmonotone}
Fig~\ref{fig:nonmonotone} shows a problem that involves in-place swapping of two objects. The problem neccesitates both the use of a buffer, as well as multiple interactions between an object and arm. This lies in the class of efficiently solvable problems that are non-monotone (and satisfy the object non-interactivity condition). 
On average using \textit{\mgcc} discovered an initial solution in $3.41s$, succeeding every time with a 8 step solution spanning $34.6s$.
\vspace{-0.3in}

\vspace{-0.05in}
\section{Discussion}
\label{sec:discussion}
\vspace{-0.1in}
The current work demonstrates the connection between Synchronized Multi-Arm Rearrangement (\moma) problems and Multi-Agent Path Finding (\mpp). The link
corresponds to an object-centric mode graph with capacity constraints, for which there are efficient solvers. These \mpp solutions are shown to be beneficial as heuristics in large scale \moma problems involving 9 arms and 4 objects, as well as a non-monotone demonstration. There are various aspects of the mode-graph that can be explored in future work. The removal of the object-non interactivity condition to encode other constraints, or dealing with mobile manipulators can lead to efficient solutions to complex non-monotone challenges. Improvements to the task planner can also allow removing the synchronization assumption.\rahul{Increasing the number of objects has a multiplicative effect on the depth of the forward search tree of actions. It is interesting to further study how larger-scale problem instances in term of the number of objects can be solved efficiently.} The current work motivates towards these intriguing avenues of future research.
\vspace{-0.15in} 
{
\vspace{-0.15in}
\bibliographystyle{spmpsci}

}

\section*{Appendix}
\subsection*{Sketch of Properties}

This section provides a sketch of the theoretical properties of the method proposed in the current work. At a high-level, the proposed approach uses an underlying framework for integrated task and motion planning based on previous work~\cite{shome2019anytime,vega2016asymptotically,Hauser2011Randomized-Multi-Modal-}. This framework retains its completeness guarantees for the proposed formulation. A contribution of the current work is to address instances of multi-arm multi-object task planning problems where a heuristic can be computed efficiently and provides useful guidance in the search process. The aforementioned heuristic is obtained by casting the object rearrangement problem to an analogous instance of a multi-agent path finding problem on a graph. Previous work~\cite{yu2016optimal} has studied this problem to demonstrate its correspondence to a multi-commodity network flow problem with its own inherent problem-class complexity.

\subsubsection*{Properties of integrated task and motion planning}
\quad The underlying task and motion planning framework outlined in the Algorithm \momaalgo is closely based on previously described integrated frameworks~\cite{shome2019anytime,vega2016asymptotically,Hauser2011Randomized-Multi-Modal-}, which:
\begin{itemize}
    \item samples \textit{transitions} - e.g, object picks, handoffs, and placements;
    \item builds a forward search tree over possible combinations of \textit{transitions} - object grasps, handoffs, and placements;
    \item and constructs a roadmap to connect each pair of consecutive transitions.
\end{itemize}

Such pipelines has been shown to be probabilistically complete~\cite{Hauser2011Randomized-Multi-Modal-}, and under certain assumptions regarding the transitions and underlying roadmaps, also asymptotically optimal~\cite{vega2016asymptotically}. Specifically, more recent analysis~\cite{shome2020pushing} studies the specific conditions for arguing asymptotic optimality for such task and motion planning algorithms.

One of the modifications when applying this principle to multi-arm problems is the decomposition of the multi-arm roadmaps into constituent roadmaps of each arm, and searching over their \textit{tensor product}~\cite{shome2019drrt}. This operation has also been proven to maintain the completeness and optimality properties of the underlying constituent roadmaps.

Algorithm \momaalgo satisfies the requirements for asymptotic optimality by sampling additional transitions for a fraction of the iterations. The \textit{select} subroutine gives every adjacent transition an opportunity to be selected in the task planning search tree, while goal biasing with the contributed heuristic. The motions between transitions operate over the \textit{tensor} structure defined over all the arms. 

\subsubsection*{Properties of the computed heuristic}
\quad The heuristic utilized to bias exploitation of \textit{actions} or \textit{transition sequences} is derived from solving a simpler variant of the problem that involves only the objects, and considers picks, handoffs and placement actions. The current work shows that this can be cast as a multi-agent path finding problem by treating the objects as agents, and their motions over a discrete graph structure defined in the current work as the capacity constrained object-centric mode graph. The current work also proves that this heuristic is action-optimal, and motion planning for the arms on top of these actions cannot decrease the number of actions involved in the solution. This renders the \mpp heuristic \textit{admissible}. 

In terms of the properties of the corresponding \mpp problem, previous work~\cite{yu2016optimal} has provided an ILP formulation that achieves \textit{complete} and efficient solutions to the NP-Hard problem instance. The current work adds capacity constraints to the formulation to make the solutions amenable to the task planning domain. It is straightforward to expand upon the original analysis for unit capacities at the vertices to argue similar complexity results. Other lines of work ~\cite{surynek2019multi} have studied capacity constraints from the point of view of SAT solvers and arrived at similar arguments. These earlier efforts guided the solution proposed in the current paper, while adhering to a linear programming framework~\cite{yu2016optimal}.

\end{document}